\documentclass[lettersize, compsoc, journal]{IEEEtran}
\usepackage{amsmath,amsfonts}
\usepackage{algorithm}
\usepackage{algorithmicx}
\usepackage{algpseudocode}
\usepackage{array}
\usepackage{textcomp}
\usepackage{stfloats}
\usepackage{url}
\usepackage{verbatim}
\usepackage{graphicx}
\usepackage[table,xcdraw,dvipsnames]{xcolor}
\usepackage{relsize}
\definecolor{mycolor}{RGB}{0,0,0}
\usepackage[
        breaklinks,
        colorlinks,
        citecolor=blue,
        linkcolor=blue,
        ]{hyperref}

\usepackage{multirow}
\usepackage{microtype}
\usepackage{booktabs} 
\usepackage{makecell}
\usepackage{bbding}
\usepackage{tcolorbox}
\usepackage{mathtools}
\usepackage{amsthm}
\usepackage{subcaption} 
\usepackage{makecell}
\usepackage{amssymb}
\newtheorem{theorem}{Theorem}
\newtheorem{definition}{Definition}

\begin{document}

\title{\LARGE \textbf{CP-uniGuard: A Unified, Probability-Agnostic, and Adaptive Framework for Malicious Agent Detection and Defense in Multi-Agent Embodied Perception Systems}}

\author{Senkang Hu, Yihang Tao, Guowen Xu, Xinyuan Qian, Yiqin Deng, Xianhao Chen, \\Sam Tak Wu Kwong,~\IEEEmembership{Fellow,~IEEE,} Yuguang Fang,~\IEEEmembership{Fellow,~IEEE}

\IEEEcompsocitemizethanks{
\IEEEcompsocthanksitem{Senkang Hu, Yihang Tao, Yiqin Deng and Yuguang Fang are with Hong Kong JC STEM Lab of Smart City and Department of Computer
Science, City University of Hong Kong, Kowloon, Hong Kong SAR. (e-mail: senkang.forest@my.cityu.edu.hk; yihang.tommy@my.cityu.edu.hk; yiqideng@cityu.edu.hk; my.Fang@cityu.edu.hk)}
\IEEEcompsocthanksitem{Guowen Xu and Xinyuan Qian are with the School of Computer Science and Engineering, University of Electronic Science and Technology of China, Chengdu, China. (e-mail: guowen.xu@uestc.edu.cn; xinyuanqian@outlook.com)}
\IEEEcompsocthanksitem{Xianhao Chen is with the Department of Electrical and Electronic Engineering, The University of Hong Kong, Pok Fu Lam, Hong Kong SAR. (e-mail: {xchen@eee.hku.hk})}
\IEEEcompsocthanksitem{Sam Tak Wu Kwong is with the School of Data Science, Lingnan University, Tuen Mun, Hong Kong SAR. (e-mail: {samkwong@ln.edu.hk})}
}

\thanks{       
The research work described in this paper was conducted in the JC STEM Lab of Smart City funded by The Hong Kong Jockey Club Charities Trust under Contract 2023-0108.  The work was supported in part by the Hong Kong SAR Government under the Global STEM Professorship and Research Talent Hub. The work of S. Hu was  supported in part by the Hong Kong Innovation and Technology Commission under InnoHK Project CIMDA. The work of G. Xu was supported in part by the National Natural Science Foundation of China under Grant No. 62502075. The work of X. Qian was supported in part by the China Postdoctoral Science Foundation under Grant BX20250389. The work of Y. Deng was supported in part by the National Natural Science Foundation of China under Grant No. 62301300. The work of X. Chen was supported in part by the Research Grants Council of Hong Kong under Grant 27213824 and Grant CRS HKU702/24. A preliminary version of this work was presented at the 39th Annual AAAI Conference on Artificial Intelligence (AAAI'25) \cite{huCPGuardMaliciousAgent2025}. (Corresponding author: Yiqin Deng)
}
}



\IEEEtitleabstractindextext{%
\begin{abstract}
Collaborative Perception (CP) has been shown to be a promising technique for multi-agent autonomous driving and multi-agent robotic systems, where multiple agents share their perception information to enhance the overall perception performance and expand the perception range. However, in CP, an ego agent needs to receive messages from its collaborators, which makes it vulnerable to attacks from malicious agents. To address this critical issue, we propose a unified, probability-agnostic, and adaptive framework, namely, CP-uniGuard, which is a tailored defense mechanism for CP deployed by each agent to accurately detect and eliminate malicious agents in its collaboration network. Our key idea is to enable CP to reach a consensus rather than a conflict against an ego agent's perception results. Based on this idea, we first develop a probability-agnostic sample consensus (PASAC) method to effectively sample a subset of the collaborators and verify the consensus without prior probabilities of malicious agents. 
Furthermore, we define collaborative consistency loss (CCLoss) for object detection task and bird's eye view (BEV) segmentation task to capture the discrepancy between an ego agent and its collaborators, which is used as a verification criterion for consensus. In addition, we propose online adaptive threshold via dual sliding windows to dynamically adjust the threshold for consensus verification and ensure the reliability of the systems in dynamic environments. Finally, we conduct extensive experiments and  demonstrate the effectiveness of our framework. Code is available at \texttt{\url{https://github.com/CP-Security/CP-uniGuard}}.
\end{abstract}

}
\maketitle

\section{Introduction}
\label{sec:intro}

\begin{figure}[t]
    \centering
    \includegraphics[width=.8\linewidth]{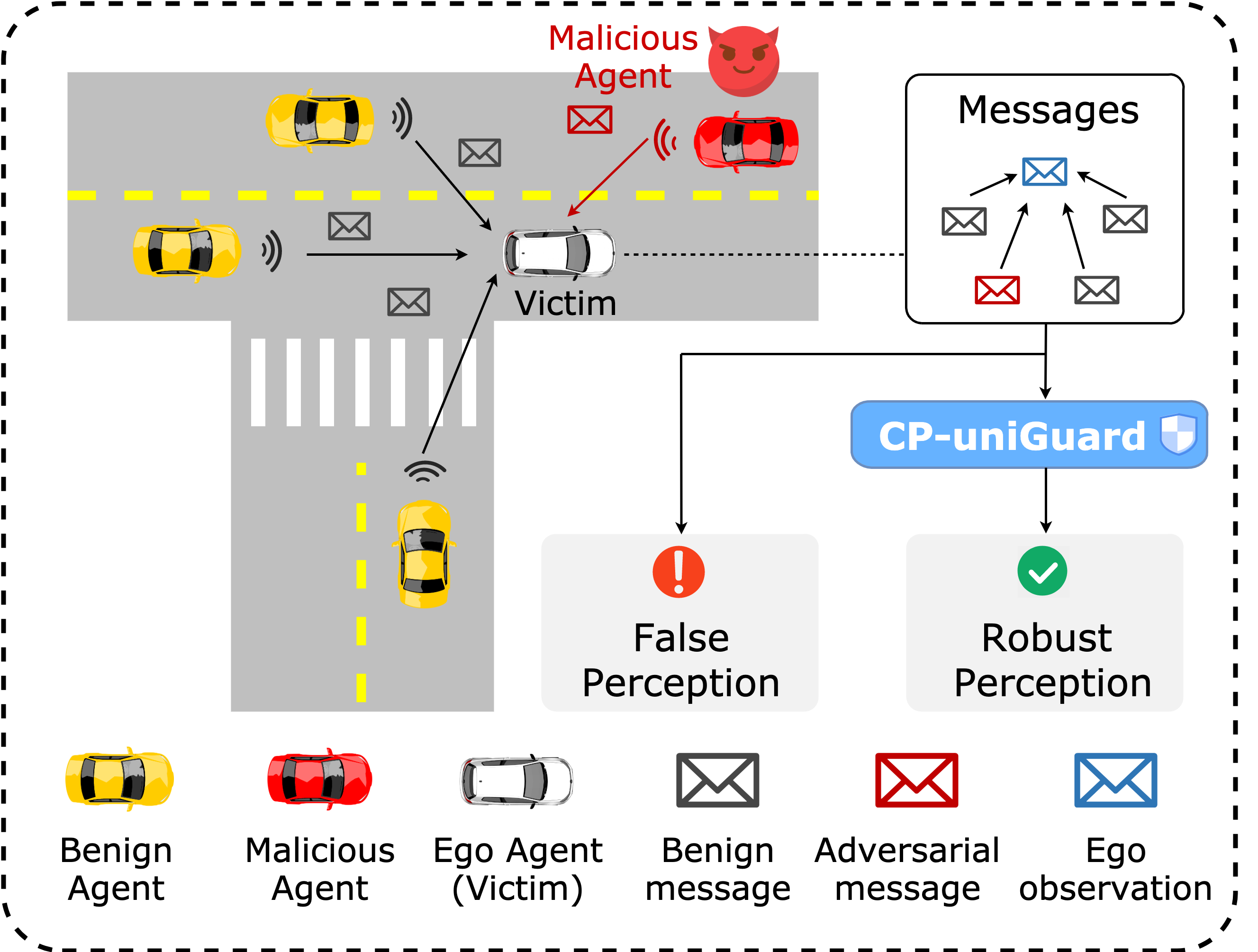}
    \caption{\textbf{Illustration of the threats of malicious agent in collaborative perception and our defense framework, CP-uniGuard}. When there is no defense, malicious agents could easily send intricately crafted adversarial messages to an ego agent, consequently misleading the CP system to yield false perception outputs. To counter this vulnerability, we propose CP-uniGuard, a tailored defense mechanism for CP to effectively detect and neutralize malicious agents, thereby ensuring robust perception outcomes.} 
    \label{fig:security_threats}
    \vspace{-6mm} 
\end{figure}

\IEEEPARstart{R}{ecently}, multi-agent collaborative perception has attracted great attention from both academia and industries since it can overcome the limitation of single-agent perception such as occlusion and limitation of sensing range \cite{hanCollaborativePerceptionAutonomous2023,huCollaborativePerceptionConnected2024}. 
Because the collaborative agents send complementary information (e.g., raw sensor data, intermediate features, and perception results) to ego agent, the ego agent can leverage this complementary information to extend its perception range and tackle the blind spot problem in its view, which is crucial for the safety of the multi-agent systems. The operational flow of CP is as follows. Each agent independently encodes local sensor inputs into intermediate feature maps. Then, these agents share their feature maps with their ego agent by agent-to-agent communication (e.g., vehicle-to-vehicle (V2V) communication in autonomous driving). Finally, the ego agent fuses the received feature maps with its own feature maps and decodes them to acquire the final perception results.
 
However, compared with single-agent perception, multi-agent CP is more vulnerable to security threats and easy to be attacked, since it incorporates the information from multiple agents, making the attack surface larger. An attack could be executed by a man-in-the-middle who alters the feature maps sent to the victim agent, or by a malicious agent that directly transmits manipulated feature maps to the victim agent. For example, Tu \textit{et al.} \cite{tuAdversarialAttacksMultiAgent2021} generated adversarial perturbations on the feature maps and attacked an ego agent, resulting in the wrong perception results. Additionally, as the encoded feature maps are not visually interpretable by humans, moderate modifications to these maps will go unnoticed, rendering the attack quite stealthy.

This issue poses significant risks to CP if an ego agent cannot accurately detect and eliminate malicious agents in its collaboration network, leading to corrupted perception results, which may result in catastrophic consequences.
For example, in autonomous driving, an ego agent may misclassify the traffic light status or fail to detect the front objects, leading to serious traffic accidents or even loss of life. Therefore, it is essential to develop a defense mechanism for CP that is robust to attack from malicious agents and can remove the malicious agents in its collaboration network.

In order to address this issue, several works have explored this problem. For example, Li \textit{et al.} \cite{liUsAdversariallyRobust2023} proposed robust collaborative sampling consensus (ROBOSAC) to randomly sample a subset of the collaborators and verify the consensus, but it requires the prior probabilities of malicious agents, which are usually unknown in practice. In addition, Zhao \textit{et al.} \cite{zhaoMaliciousAgentDetection2023} and Zhang \textit{et al.} \cite{zhangDataFabricationCollaborative2023} also developed defense methods against malicious agents, while these methods need to check the collaborators one by one, which is inefficient and computation-intensive.
Moreover, Hu \textit{et al.} \cite{hu2025cpguardnewparadigmmalicious} designed CP-Guard+, which trains an anomaly detector to identify malicious agents at the feature level. However, it requires collecting feature maps from both benign and malicious agents as training data, which is hard to achieve in real-world applications. Additionally, training an anomaly detector demands extra computational resources and lacks generalization for unseen attacks. For example, malicious agents may develop new attacks not covered by the training data to bypass the anomaly detector.

In order to overcome the aforementioned limitations, we design a novel defense mechanism, CP-uniGuard. It can be deployed by each agent to accurately detect and eliminate malicious agents in its local collaboration network. The key idea is to enable CP to achieve a consensus rather than a conflict against an ego agent's perception results.
Following this idea, we first design a \textit{probability-agnostic sample consensus} (PASAC) method to effectively sample a subset of the collaborators and verify the consensus without prior probabilities of malicious agents. In addition, the consensus is verified by our carefully designed \textit{collaborative consistency loss} (CCLoss), which is used to calculate the discrepancy between an ego agent and its collaborators. 
If a collaborator's collaborative consistency loss exceeds a certain threshold, the collaborator is considered to be a benign agent, otherwise, it is considered to be a malicious agent. In order to ensure a low false positive rate and the reliability of the system, we propose an \textit{online adaptive threshold} to dynamically adjust the threshold for consensus verification.
The main contributions of this paper are summarized as follows. 
\begin{itemize}
    \item We analyze the vulnerabilities of CP against malicious agents and develop a novel framework for robust collaborative BEV perception, CP-uniGuard, which can defend against attacks and eliminate malicious agents from a local collaboration network.
    \item We establish a probability-agnostic sample consensus (PASAC) method to effectively sample a subset of the collaborators and verify the consensus without prior probabilities of malicious agents. 
    \item In addition, we design a collaborative consistency loss (CCLoss) as a verification criterion for consensus, which can calculate the discrepancy between an ego agent and its collaborators.
    \item We also propose an online adaptive threshold scheme to dynamically adjust the threshold for consensus verification, which can ensure the reliability of the system in dynamic environments.
    \item Finally, we conduct extensive experiments on collaborative object detection and collaborative BEV segmentation tasks and demonstrate the effectiveness of our CP-uniGuard. 
\end{itemize}


\section{Background and Related Work}
\label{sec:background_related_work}

\subsection{Collaborative Perception}
\label{sec:cp}
Collaborative perception has been investigated as a means to mitigate the limitations inherent to the field-of-view (FoV) in single-agent perception systems, enhancing  accuracy, robustness, and resilience of these systems \cite{fangPACPPriorityAwareCollaborative2024, fangPrioritizedInformationBottleneck2025}. In this collaborative context, agents may opt for one of three predominant data fusion strategies: (1) early-stage raw data fusion, (2) intermediate-stage feature fusion, and (3) late-stage output fusion. Early-stage fusion, while increasing the data communication load, typically yields more precise collaboration outcomes. In contrast, late-stage fusion consumes less bandwidth but introduces greater uncertainty to the results. Intermediate-stage fusion, favored in much of the current literature, strikes an optimal balance between communication overhead and perceptual accuracy.
Research aimed at enhancing collaborative perception performance is multifaceted, addressing aspects such as communication overhead \cite{Su_Chen_Bai_Lin_Li_Qu_Zhou_2024, 9451536, 9312959}, robustness \cite{10160546,huAgentsCoDriverLargeLanguage2024,huAgentsCoMergeLargeLanguage2025,tao2024directcpdirectedcollaborativeperception,huAdaptiveCommunicationsCollaborative2024,fangPrioritizedInformationBottleneck2025}, system heterogeneity \cite{lu2024an}, and domain generalization \cite{10779389}. Among these, robustness has emerged as a particularly critical focus within the field of collaborative perception.
Despite extensive studies on system intrinsic robustness, addressing challenges such as communication disruptions \cite{10457955}, pose noise correction \cite{10160546}, and communication latency \cite{10.1007/978-3-031-19824-3_19}, most existing research works have not accounted for the presence of malicious attackers within the collaborative framework. Only a selected few studies examine the implications of robustness in scenarios compromised by malicious nodes, highlighting a significant gap in current research methodologies.

\subsection{Adversarial Perception}
Adversarial attacks targeting single-agent perception systems predominantly employ techniques such as GPS spoofing \cite{Li_2021_ICCV}, LiDAR spoofing \cite{279980}, and the deployment of physically realizable adversarial objects \cite{Tu_2020_CVPR}. In the context of multi-agent collaborative perception, the nature of adversarial strategies can vary significantly depending on the stage of collaboration.
For early-stage collaborative perception, Zhang \textit{et al.} \cite{zhangDataFabricationCollaborative2023} have developed sophisticated attacks involving object spoofing and removal. These attacks exploit vulnerabilities by simulating the presence or absence of objects and reconstructing LiDAR point clouds using advanced ray-casting techniques. In contrast, late-stage collaboration typically involves the sharing of object locations \cite{9120490}, which provides adversaries with opportunities to manipulate these shared locations easily.
Intermediate-stage attacks are particularly nuanced, often requiring that an attacker possesses white-box access to the perception models. This knowledge enables more precise manipulations of the system, though such systems are generally resistant to simplistic black-box attack strategies, such as ray-casting attacks, due to the protective effect of benign feature maps that significantly reduce the efficacy of such attacks.
Tu \textit{et al.} \cite{tuAdversarialAttacksMultiAgent2021} were among the pioneers in articulating an untargeted adversarial attack aiming at maximizing the generation of inaccurate detection bounding boxes by manipulating feature maps in intermediate-fusion systems. Building on this foundational work, Zhang \textit{et al.} \cite{zhangDataFabricationCollaborative2023} have advanced the methodology by integrating perturbation initialization and feature map masking techniques to facilitate realistic, real-time targeted attacks.
Our work is dedicated to exploring and mitigating adversarial threats specifically under intermediate-level collaborative perception framework, aiming to enhance system resilience against sophisticated attacks.

\subsection{Defensive Perception}

To fortify intermediate-level collaborative perception systems against adversarial attacks, Li \textit{et al.} \cite{liUsAdversariallyRobust2023} proposed the Robust Collaborative Sampling Consensus (ROBOSAC) method. This approach entails a random selection of a subset of collaborators for consensus verification. Despite its potential, the efficacy of ROBOSAC hinges on the availability of prior probabilities of malicious intent among agents, which are often unknown in real-world scenarios.
Moreover, Zhao \textit{et al.} \cite{zhaoMaliciousAgentDetection2023} and Zhang \textit{et al.} \cite{zhangDataFabricationCollaborative2023} have formulated defensive strategies identifying malicious agents. These techniques, however, involve scrutinizing each collaborator individually, which are both computationally intensive and inefficient.
Adversarial training has also been explored as a mechanism to bolster system robustness, as demonstrated in the studies by Tu \textit{et al.} \cite{tuAdversarialAttacksMultiAgent2021}, Raghunathan \textit{et al.} \cite{raghunathanUnderstandingMitigatingTradeoff2020}, and Zhang \textit{et al.} \cite{zhangAdversarialExamplesOpportunities2020}. While this approach enhances system robustness to address security, it substantially increases the computational load during training and may not effectively generalize to novel, unseen attacks \cite{ni2023recovering, ni2023xporter, yuan2024itpatch}. Additionally, adversarial training often results in diminished model accuracy and poses significant challenges in developing a computationally efficient and scalable adversarial defense that can be broadly applied across collaborative perception platforms.
In contrast, our approach introduces a methodology that can be autonomously implemented by each agent to accurately detect and neutralize malicious entities within the local collaborative network, aiming to enhance both the efficiency and effectiveness of defense mechanisms in collaborative perception systems.
\section{Problem Setup}
\label{sec:problem_setup}

\begin{algorithm}[t]
    \caption{Collaborative Perception by Intermediate Fusion}
    \label{alg:cp_pipeline}
    \begin{algorithmic}[1]
        \Require Inputs $\{\mathbf{x}_{i,t}\}_{i=0}^{n-1}$ from $n$ agents
        \Require Relative poses $\{\mathbf{T}_{i\to 0,t}\}_{i=1}^{n-1}$ between agents and ego agent
        \Require Encoder $E_{\theta_i}$, compressor $C_{\phi}$, aggregator $A_{\psi}$, and decoder $D_{\omega}$
        \Ensure Perception output $\hat{\mathbf{y}}_t$ (e.g., object bounding boxes or BEV segmentation map)
        \State // Stage 1: Local feature extraction for each agent
        \For{$i = 0$ to $n-1$}
            \State $\mathbf{f}_{i,t} \gets E_{\theta_i}(\mathbf{x}_{i,t})$ \Comment{Extract features}
        \EndFor
        \State // Stage 2: Coordinate transformation and message preparation
        \State $\mathtt{Mess}_t \gets \emptyset$ \Comment{Initialize message set for ego agent}
        \For{$i = 1$ to $n-1$}
            \State $\mathbf{f}_{i,t}^{w} \gets \mathcal{W}(\mathbf{f}_{i,t}, \mathbf{T}_{i\to 0,t})$ \Comment{Warp features}
            \State $\mathbf{m}_{i,t} \gets C_{\phi}(\mathbf{f}_{i,t}^{w})$ \Comment{Compress features (optional)}
            \State $\mathtt{Mess}_t \gets \mathtt{Mess}_t \cup \{\mathbf{m}_{i,t}\}$ \Comment{Add to message set}
        \EndFor
        \State // Stage 3: Feature fusion at the ego agent
        \State $\mathbf{z}_t \gets A_{\psi}(\mathbf{f}_{0,t}, \mathtt{Mess}_t)$ \Comment{Fuse the ego's and collaborators' features}
        \State // Stage 4: Generate perception output
        \State $\hat{\mathbf{y}}_t \gets D_{\omega}(\mathbf{z}_t)$ \Comment{Decode the fused features into the perception output}
        \State \Return $\hat{\mathbf{y}}_t$
    \end{algorithmic}
\end{algorithm}

\subsection{Collaborative Perception}
\label{sec:collab_perception}

As stated in Section \ref{sec:cp}, CP can be categorized into three classes according to the stage of collaboration, including early fusion, intermediate fusion, and late fusion. In this section, we focus on the widely used intermediate fusion paradigm, which is the most efficient and effective fusion strategy. 

Specifically, consider a set of $n$ agents (e.g., connected and autonomous vehicles (CAVs), unmanned aerial vehicles (UAVs), robots, etc.), with the ego agent indexed as 0. Each agent is equipped with a learnable encoder $E_{\theta_i}$, an aggregator $A_\psi$, and task-specific decoders $D_\omega$ (e.g., BEV segmentation, object detection, etc.). For the $i$-th agent, its input data is $\mathbf{x}_{i,t} \in \mathbb{R}^{W \times H \times d}$, where $W$ and $H$ are the width and height, respectively, and $d$ is the dimension of the input data. The intermediate CP pipeline can be described as follows.

\begin{enumerate}
\item \textbf{Scene state and sensor model.} Let the latent scene state at time $t$ be a random vector $\mathbf{s}_t\in\mathcal{S}$. Each agent $i\in V=\{0,1,\dots,n-1\}$ carries a sensor suite with the observation model
\begin{equation}
  \mathbf{x}_{i,t}=h_i\bigl(\mathbf{s}_t\bigr)+\boldsymbol\eta_{i,t},\qquad\boldsymbol\eta_{i,t}\sim\mathcal{N}(0,\sigma_i^2\mathbf{I}),
  \label{eq:sensor_model}
\end{equation}
where $h_i:\mathcal{S}\to\mathbb{R}^{d_i}$ encodes sensor intrinsics and mounting parameters.

\item \textbf{Local feature extraction.} Each agent passes its raw observation through a learnable encoder $E_{\theta_i}$ to obtain a local feature map
\begin{equation}
  \mathbf{f}_{i,t}=E_{\theta_i}\bigl(\mathbf{x}_{i,t}\bigr)\in\mathbb{R}^{c\times H\times W}.
  \label{eq:local_feature}
\end{equation}

\item \textbf{Coordinate transformation and message construction.} Agents' features are aligned to a common coordinate frame using the ego-centric rigid transform $\mathbf{T}_{i\to 0,t}$ from agent $i$ to the ego agent (index 0) through a 
differentiable spatial transformation 
operation $\mathcal{W}(\cdot,\mathbf{T})$. The communicated message is constructed as:
\begin{equation}
  \mathbf{m}_{i,t}=C_{\phi}\Bigl(\mathcal{W}\bigl(\mathbf{f}_{i,t},\mathbf{T}_{i\to 0,t}\bigr)\Bigr),
  \label{eq:message}
\end{equation}
where $C_{\phi}$ is an optional compression module (e.g., FP16 quantisation or learned channel pruning).

\item \textbf{Fusion graph.} Define a directed communication graph $G_t=(V,E_t)$ with $V=\{0,1,\dots,n-1\}$. Edge $(j\to i)\in E_t$ exists if agent $j$ transmits $\mathbf{m}_{j,t}$ to agent $i$ within the latency budget $\Delta\tau$. The ego sees the multiset:
\begin{equation}
  \mathtt{Mess}_t=\{\mathbf{m}_{j,t}\mid(j\to 0)\in E_t\}.
  \label{eq:fusion_graph}
\end{equation}

\item \textbf{Feature aggregation.} The ego fuses incoming features with an order-invariant aggregator $A_\psi$ (e.g., max-pool, learnable attention):
\begin{equation}
  \mathbf{z}_t=A_\psi\Bigl(\mathbf{f}_{0,t},\mathtt{Mess}_t\Bigr)\in\mathbb{R}^{c\times H\times W}.
  \label{eq:aggregator}
\end{equation}

\item \textbf{Task head and supervised objective.} A task-specific decoder $D_\omega$ produces the final prediction $\hat{\mathbf{y}}_t=D_\omega(\mathbf{z}_t)$ such as BEV segmentation maps or bounding boxes of objects. Given ground-truth labels $\mathbf{y}_t$, training minimizes the following objective function:
\begin{equation}
  \mathcal{L}(\theta,\phi,\psi,\omega)=\mathbb{E}_{t}\Bigl[\ell\bigl(\hat{\mathbf{y}}_t,\mathbf{y}_t\bigr)\Bigr]+\lambda\mathcal{R}(\theta,\phi,\psi,\omega),
  \label{eq:training_obj}
\end{equation}
where $\ell(\cdot,\cdot)$ is the task loss and $\mathcal{R}$ captures the regularization terms (e.g., L2 norm).
\end{enumerate}
The pseudo code of the collaborative perception pipeline is shown in Algorithm \ref{alg:cp_pipeline}.

\subsection{Adversarial Threat Model in Collaborative Perception}
\label{sec:threat_model_cp}

In order to defend against malicious agents in CP, we need to figure out the attack scenarios and the attacker's abilities first. Specifically, we consider an attacker to have full access to malicious agents. In addition, since the CP model is deployed on each agent, the attacker has full access to the model architecture, parameters, and intermediate feature maps, enabling the attacker to launch a white-box attack. Based on this, the attacker aims to manipulate the intermediate feature maps by adding adversarial perturbations to maximize the ego agent's perception loss. Then, these adversarial messages are transmitted to the ego agent to fool its perception fusion. The detailed attack formulation is shown in the following.

Let $V=\{0,1,\dots,n-1\}$ be the set of agents with ego index 0. We consider a subset $\mathcal{M}\subset V$ of \emph{compromised agents} with $m=|\mathcal{M}|$. For each $i\in\mathcal{M}$ we grant the adversary full knowledge of the local encoder $E_{\theta_i}$, the global aggregator $A_{\psi}$, the task decoder $D_{\omega}$, and all communication protocols, matching the worst-case white-box assumption. 

After the ego-centric warping by Eq. \eqref{eq:message}, the malicious
feature tensor of agent $i$ is \(\mathbf{f}_{i,t}^{w} \in \mathbb{R}^{c\times H\times W}\). The attacker transmits
\begin{equation}
  \widetilde{\mathbf{f}}_{i,t}^{w} = \mathbf{f}_{i,t}^{w} + \boldsymbol{\delta}_{i,t}, \quad
  \boldsymbol{\delta}_{i,t} \in \mathcal{B}_{p}(\Delta),
  \label{eq:perturbation_set}
\end{equation}
where $\boldsymbol{\delta}_{i,t}$ is the adversarial perturbation.
$\mathcal{B}_{p}(\Delta)=\{x:\lVert x\rVert_p\le \Delta\}$ is an
$\ell_p$ ball to constrain the magnitude of the perturbation. Unless stated otherwise, we adopt the imperceptibility norm
$p=\infty$ with $\Delta=0.1$ .

Let $\mathbf{y}_t$ be the ground-truth and $\ell$ the
task-specific loss. The adversarial objective is to maximize the ego agent's perception loss by manipulating the intermediate feature maps:
\begin{align}
  \max_{\{\boldsymbol{\delta}_{i,t}\}} \quad &
    \ell\Bigl(D_{\omega}\bigl(A_{\psi}(\mathbf{f}_{0,t}^{w},\dots,\mathbf{f}_{n,t}^{w} + \boldsymbol{\delta}_{n,t})\bigr), \mathbf{y}_t\Bigr), \\
  \text{s.t.}\quad & \boldsymbol{\delta}_{i,t} \in \mathcal{B}_p(\Delta),\; i\in\mathcal{M},
\label{eq:ma_attack}
\end{align}
with all other agents $(j\notin\mathcal{M})$ remaining benign.

\section{CP-uniGuard Framework}
\label{sec:framework_overview}

In this section, we present our CP-uniGuard in detail. It consists of three main components: (1) \textit{Probability-Agnostic Sample Consensus} (PASAC), (2) \textit{Collaborative Consistency Loss Verification} (CCLoss), and (3) \textit{Online Adaptive Threshold}. PASAC is designed to effectively sample a subset of collaborators for consensus verification without relying on prior probabilities of malicious intent. CCLoss is proposed to verify the consensus between the ego agent and the collaborative agents. Online Adaptive Threshold is leveraged to dynamically adjust the threshold for consensus verification.
These three components work collaboratively to detect and neutralize malicious agents in the local collaboration network. We elaborate on these three components in the following sections.

\begin{algorithm}[t]
    \caption{PASAC}
    \label{alg:pasac}
    \textbf{Input}: 
    \begin{itemize}
        \item $\{\mathbf{f}_{i,t}^w\}_{i=1}^{n-1}$, warped intermediate feature maps from collaborators. $\mathbf{f}_{0,t}$, the intermediate feature of ego agent.
        \item $A_{\psi}$, $D_{\omega}$, the aggregator and decoder.
        \item $N_{\max}$, the maximum number of selected collaborators for ego agent, and $N_{\max}\leq n-1$.
        \item $\varepsilon$, the threshold of $\mathcal{L}_\mathrm{CCLoss}$.
    \end{itemize}
    \textbf{Output}: $\{\mathbf{B}_i\}$, the set of benign collaborators
    \begin{algorithmic}[1] 
    \State $\hat{\mathbf{y}}_{0,t} \gets D_{\omega}(\mathbf{f}_{0,t})$ 
    \Procedure{\texttt{PASAC}}{$\{\mathbf{f}_{i,t}^w\}_{i=1}^{n-1}$} 
        \If{$\texttt{len}(\{\mathbf{B}_i\})\geq N_{\max}$} \Comment{Check if enough benign agents found}
        \State \Return $\{\mathbf{B}_i\}$
        \EndIf
        \If{$\texttt{len}(\{\mathbf{f}_{i,t}^w\})=1$} 
        \State $\hat{\mathbf{y}}_{k,t} \gets D_{\omega}(A_{\psi}(\mathbf{f}_{0,t}, \mathbf{f}_{k,t}^w))$ 
            \If{$\mathcal{L}_\mathrm{CCLoss}(\hat{\mathbf{y}}_{0,t}, \hat{\mathbf{y}}_{k,t})\leq\varepsilon$} \Comment{Check benign}
            \State $\{\mathbf{B}_i\} \leftarrow \{\mathbf{B}_i\}\cup\mathbf{f}_{k,t}^w$ 
            \EndIf
            \State \Return $\{\mathbf{B}_i\}$
        \EndIf
        \State $\{\mathbf{f}_{i,t}^w\}_{i=1}^{n-1}\rightarrow\{\mathbf{f}_{i,t}^w\}_{i\in \mathrm{G}_1}, \{\mathbf{f}_{i,t}^w\}_{i\in \mathrm{G}_2}$ 
        \State $\hat{\mathbf{y}}_{\mathrm{G}_1,t} \gets D_{\omega}(A_{\psi}(\mathbf{f}_{0,t}, \{\mathbf{f}_{i,t}^w\}_{i\in \mathrm{G}_1}))$ \Comment{Fusion}
        \State $\hat{\mathbf{y}}_{\mathrm{G}_2,t} \gets D_{\omega}(A_{\psi}(\mathbf{f}_{0,t}, \{\mathbf{f}_{i,t}^w\}_{i\in \mathrm{G}_2}))$ \Comment{Fusion}
        \If{$\mathcal{L}_\mathrm{CCLoss}(\hat{\mathbf{y}}_{0,t}, \hat{\mathbf{y}}_{\mathrm{G}_1,t})\leq\varepsilon$} 
        \State $\{\mathbf{B}_i\}_\mathrm{sublist} \gets \texttt{PASAC}(\{\mathbf{f}_{i,t}^w\}_{i\in \mathrm{G}_1})$ 
        \State $\{\mathbf{B}_i\}\leftarrow \{\mathbf{B}_i\}\cup \{\mathbf{B}_i\}_\mathrm{sublist}$ 
        \Else \Comment{First group contains malicious agents}
        \State $\{\mathbf{B}_i\}\leftarrow \{\mathbf{B}_i\}\cup \{\mathbf{f}_{i,t}^w\}_{i\in \mathrm{G}_1}$ 
        \EndIf 
        \If{$\mathcal{L}_\mathrm{CCLoss}(\hat{\mathbf{y}}_{0,t}, \hat{\mathbf{y}}_{\mathrm{G}_2,t})\leq\varepsilon$} 
        \State $\{\mathbf{B}_i\}_\mathrm{sublist} \gets \texttt{PASAC}(\{\mathbf{f}_{i,t}^w\}_{i\in \mathrm{G}_2})$ 
        \State $\{\mathbf{B}_i\}\leftarrow \{\mathbf{B}_i\}\cup \{\mathbf{B}_i\}_\mathrm{sublist}$ 
        \Else \Comment{Second group contains malicious agents}
        \State $\{\mathbf{B}_i\}\leftarrow\{\mathbf{B}_i\}\cup\{\mathbf{f}_{i,t}^w\}_{i\in \mathrm{G}_2}$ 
        \EndIf
        \State \Return $\{\mathbf{B}_i\}$ 
    \EndProcedure
    \end{algorithmic}
\end{algorithm}

\newcommand{\Test}{\text{Test}}
\newcommand{\BENIGN}{\mathtt{BENIGN}}
\newcommand{\CONTAM}{\mathtt{CONTAM}}

\subsection{PASAC: Probability-Agnostic Sample Consensus}
\label{sec:pasac_theory}

\subsubsection{Our Method}

To achieve the consensus of agents, the most straightforward method is to check the agents one by one. However, this method is time-consuming and computation-intensive, especially when the number of agents is large. A better method is to randomly sample a subset of agents for consensus verification at each time, such as ROBOSAC \cite{liUsAdversariallyRobust2023}. However, ROBOSAC requires the prior probabilities of malicious intent among agents, which are often unknown in real-world scenarios. To fill in this research gap, we propose PASAC.

Specifically, consider the set $V=\{0,1,\dots,n-1\}$ of collaborative agents. An unknown subset $\mathcal{M}\subset V$ (``malicious agents'') of cardinality $m=|\mathcal{M}|$ may behave arbitrarily so as to degrade collaborative perception. Denote by $\mathcal{B}=V\setminus\mathcal{M}$ the benign agents.
Given any query subset $S\subseteq V$, the aggregator performs a group-consistency test 
\begin{equation}
  \label{eq:def_test}
  \Test(S)\in\{\BENIGN,\CONTAM\},
\end{equation}
where $\Test(S)$ acts as a noisy oracle that reveals whether $S$ intersects $\mathcal{M}$. Here, $\BENIGN$ and $\CONTAM$ indicate the absence or presence of malicious agents in $S$, respectively.

\begin{definition}[$(\alpha,\beta)$-reliability]
  \label{def:reliable}
  A test oracle formulated in Eq.~\eqref{eq:def_test} is called $(\alpha,\beta)$-reliable if, for every $S\subseteq V$,
  \begin{align*}
    &\Pr\bigl[\Test(S)=\BENIGN\bigm| S\cap\mathcal{M}=\varnothing\bigr]
      \ge1-\alpha,\\[2pt]
    &\Pr\bigl[\Test(S)=\CONTAM\bigm| S\cap\mathcal{M}\neq\varnothing\bigr]
      \ge1-\beta.
  \end{align*}
  Soundness corresponds to false-positive rate $\alpha$, while completeness corresponds to false-negative rate $\beta$.
\end{definition}
In practice, choosing the CCLoss threshold $\varepsilon$ yields $(\alpha,\beta)<0.05$.

PASAC is an \emph{adaptive binary-splitting} procedure that examines ever smaller subsets until a predefined quota of benign agents has been certified.

The procedure is initialized with $(S,k)=(V, N_{\max})$, where $N_{\max}$ is a user-defined upper bound on the number of desired benign peers (e.g., $N_{\max}=10$). 
During the recursive splitting process, the maximum depth of any branch is $\lceil\log_2|S|\rceil$, which means that each malicious agent can cause at most $\lceil\log_2|S|\rceil$ $\CONTAM$ labels before being isolated.

More specifically, the ego agent will generate the collaborative perception results $\mathbf{y}_t$ based on its observation and the received messages for feature fusion from the $n-1$ other agents. Firstly, the ego agent generates its perception results $\mathbf{y}_{0,t}$ based on its own observation. Then, it randomly splits the agents into two groups $\mathrm{G}_1$ and $\mathrm{G}_2$ of equal size \textcolor{mycolor}{(nearly equal when the number of agents is odd)}. 
After receiving all the messages, the ego agent fuses the features and generates the perception results $\mathbf{y}_{\mathrm{G}_1,t}$ based on the messages from the first group and $\mathbf{y}_{\mathrm{G}_2,t}$ based on the messages from the second group.

Then, the ego agent verifies the consensus and checks if there is any malicious agent in the two groups. The consensus is verified by CCLoss to be introduced in the next section. Specifically, the CCLoss is calculated between the ego agent and each group, that is, $\mathcal{L}_\mathrm{CCLoss}(\mathbf{y}_{0,t}, \mathbf{y}_{\mathrm{G}_1,t})$ and $\mathcal{L}_\mathrm{CCLoss}(\mathbf{y}_{0,t}, \mathbf{y}_{\mathrm{G}_2,t})$. If the CCLoss exceeds a certain threshold, the group is considered benign, otherwise, it is considered to contain malicious agents. In addition,  suppose the first group is benign and the second group is verified to have malicious agents, all agents in the first group are marked as benign and incorporated in the following collaboration. For the second group, the ego agent continues to split the second group into two subgroups and repeats the consensus verification process. This process will continue until finding all the malicious agents or obtaining enough benign agents. The detailed procedures of PASAC is shown in Algorithm~\ref{alg:pasac}.

\subsubsection{Theoretical Analysis}
In this subsection, we will carry out theoretical analysis on our proposed scheme.

\paragraph{\textbf{Success Probability}} We first give the theoretical result on the success probability of PASAC.
\begin{theorem}[Success probability of PASAC]
  \label{thm:pasac_success}
  Let $T\bigl(n,m\bigr)$ be the (random) number of oracle queries incurred by
  Algorithm~\ref{alg:pasac} with $S=[n]$ and $k=N_{\max}\ge n-m$.
  Under an $(\alpha,\beta)$-reliable oracle, PASAC mis-classifies at most one
  agent with probability
  \begin{equation}
    \Pr[\text{error}]\le(\alpha+\beta)\,m\,\lceil\log_2 n\rceil.
    \label{eq:err_prob}
  \end{equation}
\end{theorem}

\begin{proof}
  Fix an arbitrary malicious agent $j\in\mathcal{M}$. Along the unique branch
  that isolates $j$, the subset size halves at each level until $\{j\}$ is
  reached, hence at most $\lceil\log_2 n\rceil$ nodes on that branch are
  labelled $\CONTAM$. By the union bound and Definition~\ref{def:reliable}, the
  probability that \emph{any} of those nodes is mis-labelled does not exceed
  $(\alpha+\beta)\,\lceil\log_2 n\rceil$ for each individual malicious agent. Applying the union bound over
  all $m$ malicious agents (summing their individual error probabilities) yields the total error probability in Eq.~\eqref{eq:err_prob}.
\end{proof}
Equation~\eqref{eq:err_prob} implies that setting $(\alpha+\beta)<0.01$ and $n\le100$ ensures overall error probability below~$\smash{\approx2\%}$ even
with $m=10$ attackers.

\paragraph{\textbf{Query-Complexity Bounds}}
In the worst-case scenario, we disregard the early stopping condition $k=n-m$, which is met when all benign agents are identified. We define the positive dyadic-logarithm $\bigl\lceil\log_2 x\bigr\rceil_+$ as $\max(0, \lceil\log_2 x\rceil)$.

\begin{theorem}[Upper bound]
  \label{lem:upper_complexity}
  The total number of oracle queries satisfies
  \begin{equation}
    \label{eq:upper_T}
    T(n,m)\le 2m\bigl\lceil\log_2 n\bigr\rceil + (n-m).
  \end{equation}
\end{theorem}

\begin{proof}
We analyse the recursion tree $\mathcal{T}$ implicitly generated by PASAC.
Each node corresponds to a subset $S\subseteq V$ that is submitted to the
oracle $\Test(\cdot)$, hence one oracle query per node.

Nodes are marked as $\CONTAM$ if $S\cap\mathcal{M}\neq\varnothing$ and
$\BENIGN$ otherwise.
Once a node is $\BENIGN$, it is never expanded again, so it is a leaf of
$\mathcal{T}$.
Conversely, every $\CONTAM$ node is split into two (roughly) equal-size
children according to the algorithm.

Let $L \triangleq \lceil\log_2 n\rceil$.  
Fix a malicious agent $j\in\mathcal{M}$ and follow the unique
root-to-leaf path that eventually isolates $\{j\}$.
Starting from $|V|=n$, each split at most halves the size of the contaminated
subset that contains $j$.
After $L$ levels the subset size is $<2m$; at most another $L$ splits are
needed before every malicious singleton appears as a leaf.
Hence, \emph{the depth of any contaminated path is bounded by~$L$}.  
There are $m$ such root-to-leaf paths (one per malicious agent).  
Each path contains at most $L$ contaminated nodes.
Moreover, when a contaminated node $S$ is split, at most one \emph{additional}
contaminated sibling (the "other half'') can be created.
Therefore, the total number of contaminated internal nodes satisfies $|\textsc{Contam‐int}| \;<\; 2mL$.

Every benign singleton $\{i\}$ with $i\notin\mathcal{M}$ is queried exactly
once and produces a benign leaf.
Hence there are precisely $n-m$ such leaves.
Adding the two contributions yields
\begin{equation}
  \begin{aligned}
    T(n,m)\;\le & \; 2mL \;+\; (n-m) \\
     =&2m\bigl\lceil\log_2 n\bigr\rceil \;+\; (n-m),
  \end{aligned}
\end{equation}
which is exactly Eq.~\eqref{eq:upper_T}.
\end{proof}

\subsection{Collaborative Consistency Loss Verification}
\label{sec:multi_test_collaborative_consistency_loss_verification}

To verify the consensus between the ego agent and the collaborative agents, we design a novel loss function, \textit{Collaborative Consistency Loss} (CCLoss), which is used to calculate the discrepancy between the ego agent and the collaborative agents and verify the consensus. It consists of two parts for object detection and BEV segmentation tasks, respectively.

\subsubsection{Collaborative Consistency Loss for Object Detection}

Given the intermediate feature maps $\mathbf{f}_{0,t}$ of the ego agent and a set of intermediate feature maps $\{\mathbf{f}_{1,t}, \ldots, \mathbf{f}_{i,t}\}$ from the collaborative agents at time $t$. The ego agent can generate two object detection results (bounding boxes proposals) by object detection decoder $D_{\omega}^{(\mathrm{det})}$:
\begin{equation}
    \hat{\mathbf{y}}_{0,t}=D_{\omega}^{(\mathrm{det})}(\mathbf{f}_{0,t}),
\end{equation}
\begin{equation}
    \hat{\mathbf{y}}_{\mathrm{fuse},t}=D_{\omega}^{(\mathrm{det})}(A_{\psi}(\mathbf{f}_{0,t}, \mathbf{f}_{1,t}, \ldots, \mathbf{f}_{i,t})),
\end{equation}
where $\hat{\mathbf{y}}_{0,t}=\{{y}^{(1)},\cdots,{y}^{(L)}\}$ and $\hat{\mathbf{y}}_{\mathrm{fuse},t}=\{{y}_{\mathrm{fuse}}^{(1)},\cdots,{y}_{\mathrm{fuse}}^{(L)}\}$ are the bounding box proposals of the ego agent and the collaborative agents, respectively.


Let $I_c$ and $I_{\mathrm{fuse},c}$ be the two sets of indices of all the proposed bounding boxes in $\hat{\mathbf{y}}_{0,t}$ and $\hat{\mathbf{y}}_{\mathrm{fuse},t}$ that are predicted to some class $c \in \mathcal{C}$, respectively. Let $\mathfrak{S}_c$ be all matches between $I_c$ and $I_{\mathrm{fuse},c}$ that maps each $l \in I_c$ to a unique $l' \in I_{\mathrm{fuse},c}$. Note that to guarantee such uniqueness when $|I_{\mathrm{fuse},c}| < |I_c|$, we pad $I_{\mathrm{fuse},c}$ to the same size as $I_c$ with arbitrary ``dummy'' indices not appeared in $\{1,\cdots,L\}$ and associate each ``dummy'' index with an empty box $\varnothing$. Then, the CCLoss between $\hat{\mathbf{y}}_{0,t}$ and $\hat{\mathbf{y}}_{\mathrm{fuse},t}$ is defined by:
\begin{equation}
  \begin{aligned}
  &\mathcal{L}_\mathrm{CCLoss}^{(\mathrm{det})}(\hat{\mathbf{y}}_{0,t}, \hat{\mathbf{y}}_{\mathrm{fuse},t}) \\
  &= 1-\frac{1}{|\mathcal{C}|}\sum_{c\in\mathcal{C}} \min_{\sigma\in\mathfrak{S}_c} \frac{1}{|I_c|} \sum_{l\in I_c} \mathcal{L}_\mathrm{box}({y}^{(l)},{y}_{\mathrm{fuse}}^{(\sigma(l))};c),
  \end{aligned}
  \label{eq:cc_loss_box}
\end{equation}
where the optimal match for each class $c$ is solved using the Hungarian match algorithm \cite{kuhn1955hungarian}. Here, for any class $c \in \mathcal{C}$ and any two arbitrary bounding box proposals $\mathbf{y}$ and $\mathbf{y}'$,
\begin{equation}
  \begin{aligned}
  & \mathcal{L}_\mathrm{box}(\mathbf{y},\mathbf{y}';c) \\
  &= \frac{1}{1+\phi}\left(\max(p_c - p'_c,0) + \phi(1-\mathrm{IoU}(\mathbf{y},\mathbf{y}'))\right),
  \label{eq:cc_loss_box_def}
  \end{aligned}
\end{equation}
where $p_c$ and $p'_c$ represent the posteriors of class $c$ associated with bounding boxes $\mathbf{y}$ and $\mathbf{y}'$, respectively, and $\phi$ is the parameter balancing the two terms. Note that if $\mathbf{y}' = \varnothing$ is an empty box (e.g., due to the padding), the equation above will be reduced to $\mathcal{L}_\mathrm{box}(\mathbf{y},\mathbf{y}';c) = p_c + \phi$, since $p'_c$ and $\mathrm{IoU}(\mathbf{y},\mathbf{y}')$ are both zero.

From Eq.~\eqref{eq:cc_loss_box}, we can see that when the two bounding box proposals are consistent, the value of $\mathcal{L}_\mathrm{CCLoss}^{(\mathrm{det})}$ will be close to 1. On the contrary, if these two proposals are different, the value $\mathcal{L}_\mathrm{CCLoss}^{(\mathrm{det})}$ will be close to 0, which means there may be a malicious agent in the collaborative agents.

\subsubsection{Collaborative Consistency Loss for BEV Segmentation}

Given the intermediate feature maps $\mathbf{f}_{0,t}$ of the ego agent and a set of intermediate feature maps $\{\mathbf{f}_{1,t}, \ldots, \mathbf{f}_{i,t}\}$ from the collaborative agents at time $t$. The ego agent can generate two BEV segmentation maps by BEV segmentation decoder $D_{\omega}^{(\mathrm{seg})}$:
\begin{equation}
    \hat{\mathbf{y}}_{0,t}=D_{\omega}^{(\mathrm{seg})}(\mathbf{f}_{0,t}),
\end{equation}
\begin{equation}
    \hat{\mathbf{y}}_{\mathrm{fuse},t}=D_{\omega}^{(\mathrm{seg})}(A_{\psi}(\mathbf{f}_{0,t}, \mathbf{f}_{1,t}, \ldots, \mathbf{f}_{i,t})),
\end{equation}
where $\hat{\mathbf{y}}_{0,t}$ and $\hat{\mathbf{y}}_{\mathrm{fuse},t}$ are 3D matrices and their sizes are in $\mathbb{R}^{W_D\times H_D\times C}$ with $W_D$, $H_D$, and $C$ being the width, height, and the number of classes of the BEV segmentation map, respectively. 

Following the idea that enables CP to achieve consensus rather than conflict with the ego agent's perception result, we carefully design the CCLoss for BEV segmentation taskto measure the discrepancy between the ego agent and the collaborative agents, which is formulated as:
\begin{equation}
    \begin{aligned}
        &\mathcal{L}_\mathrm{CCLoss}^{(\mathrm{seg})}(\hat{\mathbf{y}}_{0,t}, \hat{\mathbf{y}}_{\mathrm{fuse},t})=\\
        &\frac{\sum^{|\mathcal{C}|}_{j=1}w_j\sum^{W_D\cdot H_D}_{i=1}p^0_{i,j}p^\mathrm{fuse}_{i,j}}{\sum^{|\mathcal{C}|}_{j=1}w_j\left(\sum^{W_D\cdot H_D}_{i=1}p^0_{i,j} + \sum^{W_D\cdot H_D}_{i=1}p^\mathrm{fuse}_{i,j}\right)}
    \end{aligned}
\end{equation}
where $c$ is the number of classes, $p^0_{i,j}$ and $p^\mathrm{fuse}_{i,j}$ are the probabilities of the $j$-th class at the $i$-th pixel in the BEV segmentation map $\hat{\mathbf{y}}_{0,t}$ and $\hat{\mathbf{y}}_{\mathrm{fuse},t}$, respectively, and $w_j$ is the weight of the $j$-th class, defined as the inverse frequency of the class $w_j=1/(\sum^{|\mathcal{C}|}_{j=1}(p^0_{i,j}+ p^\mathrm{fuse}_{i,j}))^2$. For the numerator of $\mathcal{L}_\mathrm{CCLoss}^{\mathrm{seg}}$, it calculates the weighted sum of the product of the probabilities for each pixel and each class, which essentially measures the overlap between the two distributions. The weight $w_j$ ensures that the contribution of each class is adjusted according to its importance or frequency. The denominator sums up the weighted sums of the probabilities from both the ego agent's prediction map and the fused segmentation maps for each class. It represents the total probability mass for each class, adjusted by the weights. Finally, the fraction measures the similarity between the two distributions. If these two distributions are similar, the value of $\mathcal{L}_\mathrm{CCLoss}^{\mathrm{seg}}$ will be close to 1. On the contrary, if these two distributions are different, the value $\mathcal{L}_\mathrm{CCLoss}{\mathrm{seg}}$ will be close to 0, which means there may be a malicious agent in the collaborators.

\subsection{Online Adaptive Threshold via Dual Sliding Windows}
\label{sec:adaptive-threshold}

To maintain $(\alpha,\beta)$‑reliability under dynamic environment, we propose an online adaptive threshold mechanism via dual sliding windows. We keep \emph{two} fixed-length sliding windows: one stores the most recent scores labelled
$\BENIGN$, the other stores those labelled $\CONTAM$.  At every frame we take
the upper $(1-\alpha)$‑quantile of the benign window and the lower
$\beta$‑quantile of the contaminated window, average them, and obtain a
provisional threshold.  An exponentially weighted moving average (EWMA) then converts this provisional value into the working threshold $\varepsilon_t$ used for classification.

Specifically, let $Z_t = \mathcal{L}_\mathrm{CCLoss}\bigl(\hat{\mathbf{y}}_{0,t},
  \hat{\mathbf{y}}_{\text{fuse},t}\bigr)\in[0,1]$
be the per-frame consistency score (larger means better agreement).
After the oracle labels the current subset $S_t$, we append $Z_t$ to the
corresponding deque:
\begin{equation}
  \mathcal{W}_t^{P}\;(\BENIGN),\ 
  \mathcal{W}_t^{N}\;(\CONTAM),
\end{equation}
each of the maximum length $W$.

For any window $\mathcal{W}$ of size $n$, denote the empirical
$q$‑quantile by
$
  \hat q_{q}(\mathcal{W})
  =\inf\bigl\{z: \tfrac1n\sum_{Z\in\mathcal{W}}\!\mathbf 1\{Z\le z\}\ge q\bigr\}.
$
We use the \emph{upper} $(1-\alpha)$‑quantile of benign scores and the
\emph{lower} $\beta$‑quantile of contaminated scores:
\begin{equation}
  \hat q_{1-\alpha}^{P}(t)=\hat q_{1-\alpha}\!\bigl(\mathcal{W}_t^{P}\bigr),
  \quad
  \hat q_{\beta}^{N}(t)=\hat q_{\beta}\!\bigl(\mathcal{W}_t^{N}\bigr).
\end{equation}
Then, we can calculate the provisional threshold $\tilde\varepsilon_t$ as the mid‑point:
\begin{equation}
  \tilde\varepsilon_t
  =\frac{\hat q_{1-\alpha}^{P}(t)+\hat q_{\beta}^{N}(t)}{2}.
  \label{eq:midpoint-thr}
\end{equation}
In addition, we apply an EWMA update to smooth the provisional threshold:
\begin{equation}
  \varepsilon_t=(1-\eta)\,\varepsilon_{t-1}+\eta\,\tilde\varepsilon_t,
  \qquad 0<\eta\le1.
  \label{eq:smooth}
\end{equation}
Then, we can classify the current subset $S_t$ as benign or contaminated based on the following rule:
\begin{equation}
  \Test(S_t)=
  \begin{cases}
    \BENIGN, & \text{if } Z_t\ge\varepsilon_t,\\[2pt]
    \CONTAM, & \text{otherwise}.
  \end{cases}
  \label{eq:test-rule-dual}
\end{equation}

\begin{algorithm}[t]
  \caption{Dual‑Window Online Threshold Update}
  \label{alg:dual-threshold}
  \begin{algorithmic}[1]
    \Require New score $Z_t$, decision $\hat y_t\in\{\BENIGN,\CONTAM\}$; windows $\mathcal W^P,\mathcal W^N$; parameters $(\alpha,\beta,W,\eta)$.
    \If{$\hat y_t=\BENIGN$} \State push\_back$(\mathcal W^P,Z_t)$
    \Else \State push\_back$(\mathcal W^N,Z_t)$ \EndIf
    \State truncate both windows to length $W$
    \If{$|\mathcal W^P|\!\ge\!W_{\min}$ \textbf{and} $|\mathcal W^N|\!\ge\!W_{\min}$}
      \State compute $\hat q_{1-\alpha}^{P},\hat q_{\beta}^{N}$
      \State $\tilde\varepsilon\gets(\hat q_{1-\alpha}^{P}+\hat q_{\beta}^{N})/2$
      \State $\varepsilon\gets(1-\eta)\varepsilon+\eta\,\tilde\varepsilon$
    \EndIf
    \State \Return $\varepsilon$
  \end{algorithmic}
\end{algorithm}

\begin{theorem}[Reliability guarantee]
  \label{thm:reliability}
Assume that within each window, benign (resp.\ contaminated) scores are
i.i.d.\ with cumulative distribution functions (CDFs) $F_P$ (resp.\ $F_N$) such that
$F_P^{-1}(1-\alpha)>F_N^{-1}(\beta)$.  If
$|\mathcal W_t^{P}|,|\mathcal W_t^{N}|\ge W_{\min}$,
then
\begin{align}
  &\Pr\bigl[Z_t\ge\varepsilon_t\mid S_t\cap\mathcal M=\varnothing\bigr]
    \ge 1-\alpha,\\
  &\Pr\bigl[Z_t<\varepsilon_t\mid S_t\cap\mathcal M\neq\varnothing\bigr]
    \ge 1-\beta.
\end{align}
\end{theorem}

\begin{proof}
Let $q_P = F_P^{-1}(1-\alpha)$, $q_N = F_N^{-1}(\beta)$, and $\varepsilon_t = (q_P + q_N)/2$.
Under the assumption $q_N < q_P$, we have $q_N < \varepsilon_t < q_P$.

\textbf{Benign case.} If $S_t\cap\mathcal M=\varnothing$, then $Z_t\sim F_P$. Because $F_P$ is non‑decreasing and $\varepsilon_t \le q_P$, we have
\begin{equation}
  \begin{aligned}
    &F_P(\varepsilon_t) \le F_P(q_P) = \alpha \\
    & \Longrightarrow 
    \Pr\bigl[Z_t \ge \varepsilon_t\bigr] = 1 - F_P(\varepsilon_t) \ge 1 - \alpha.
  \end{aligned}
\end{equation}

\textbf{Contaminated case.} If $S_t\cap\mathcal M\neq\varnothing$, then $Z_t\sim F_N$. Since $\varepsilon_t \ge q_N$, we have
\begin{equation}
  \begin{aligned}
    &F_N(\varepsilon_t) \ge F_N(q_N) = 1 - \beta \\
    & \Longrightarrow 
    \Pr\bigl[Z_t < \varepsilon_t\bigr] = F_N(\varepsilon_t) \ge 1 - \beta.
  \end{aligned}
\end{equation}

Combining the two cases, we complete the proof.
\end{proof}

Alg.~\ref{alg:dual-threshold} details the online update. This dual‑window rule adapts rapidly to changing attack strategies and various environments.

\section{Experiments}
\label{sec:experiments}

\begin{figure*}[t]
    \centering
    \includegraphics[width=.9\linewidth]{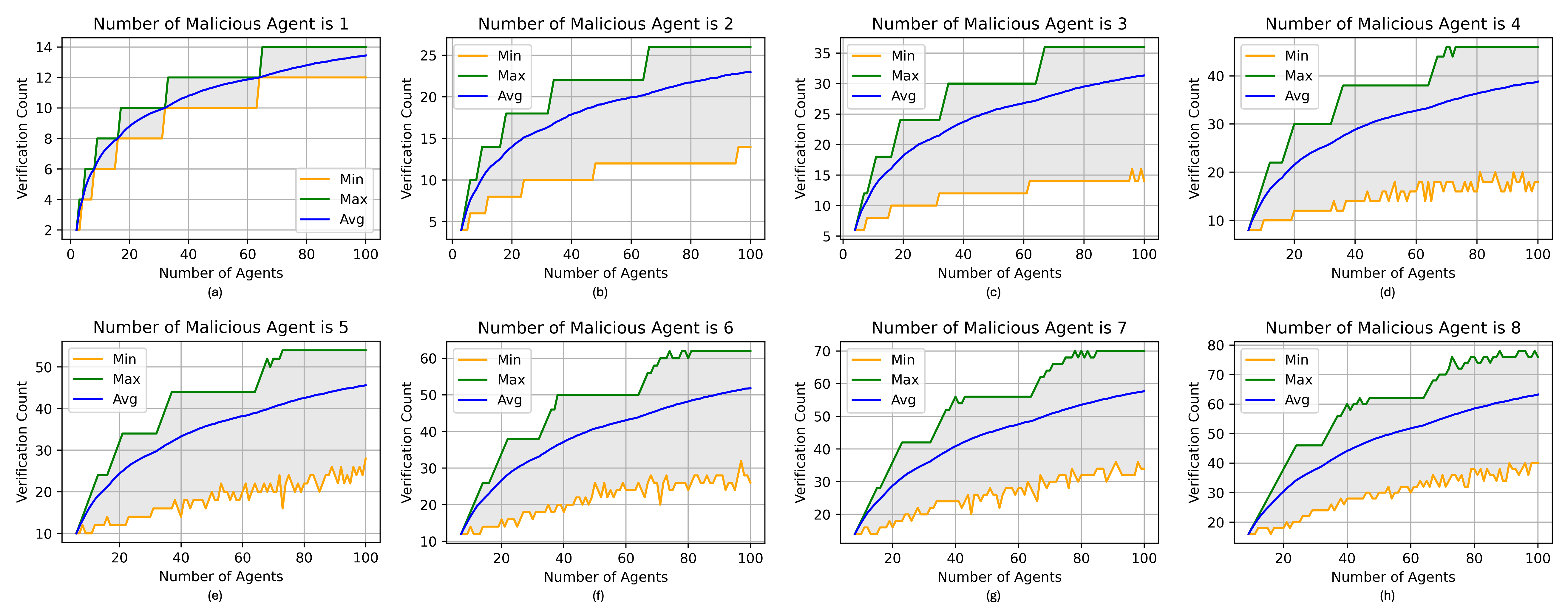}
    \vspace{-3mm}
    \caption{\textbf{Quantitative results of PASAC:} Number of Overall Agents vs Verification Count.}
    \label{fig:pasac_analysis}
    \vspace{-3mm}
\end{figure*}

\begin{figure*}[t]
    \centering
    \includegraphics[width=.9\linewidth]{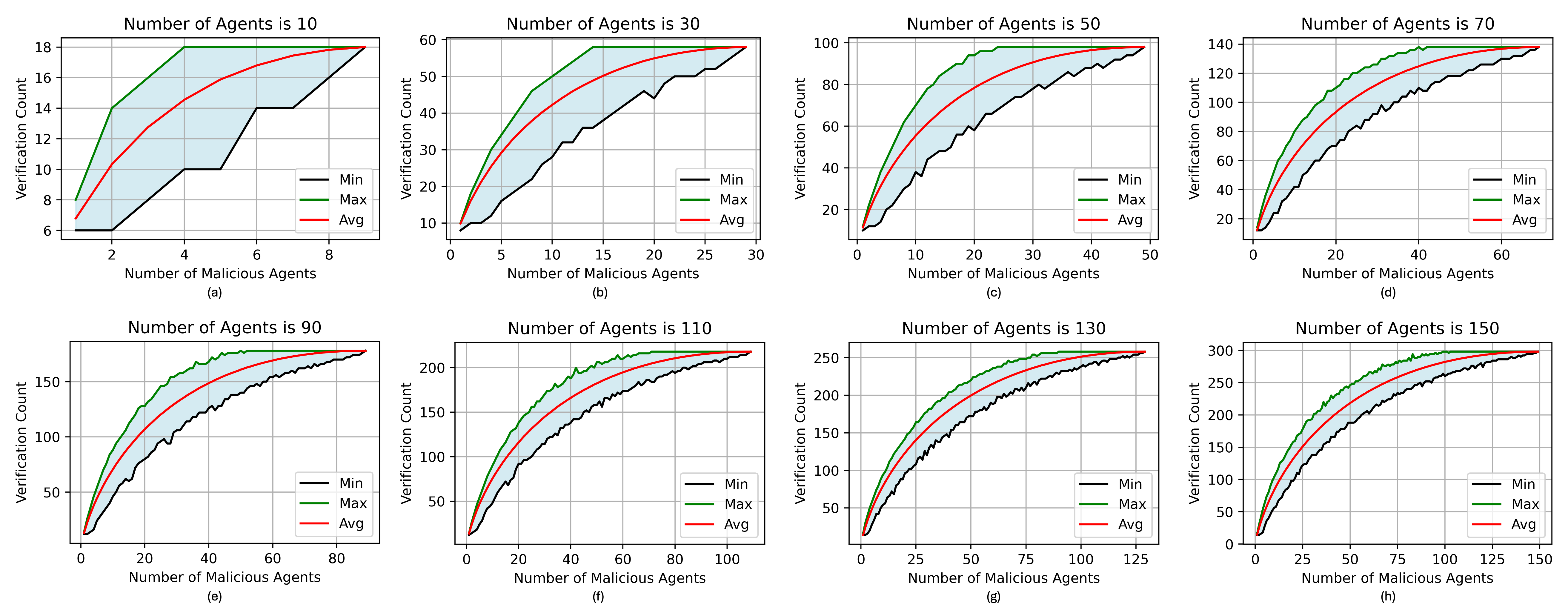}
    \vspace{-3mm}
    \caption{\textbf{Quantitative results of PASAC:} Number of Malicious Agents vs Verification Count.}
    \label{fig:pasac_analysis_2}
    \vspace{-3mm}
\end{figure*}

\begin{table*}[t]
    \caption{\textbf{Quantitative results on BEV segmentation task in V2X-Sim dataset.} Upper-bound denotes collaborative perception with all benign agents. Lower-bound is individual perception.}
    \vspace{-3mm}
    \label{tab:quantitative_results}
    \resizebox{1\linewidth}{!}{
    \begin{tabular}{p{5cm}|ccccccc|c}
        \toprule        
        { \textbf{Method} }& {Vehicle} & {Sidewalk} & {Terrain} & {Road} & {Buildings} & {Pedestrian} & {Vegetation} & \textbf{mIoU} \\
        \midrule
        Upper-bound & 55.58 & 48.20 & 47.33 & 69.60 & 29.34 & 21.67 & 41.02 & 40.45 \\
        CP-uniGuard (against FGSM attack) & 52.76 & 46.35 & 46.67 & 68.32 & 28.98 & 20.51 & 40.15 & 39.30 \\
        CP-uniGuard (against C\&W attack) & 49.22 & 44.08 & 44.76 & 65.58 & 30.12 & 20.83 & 39.10 & 37.95 \\
        CP-uniGuard (against PGD attack) & 52.84 & 46.41 & 46.73 & 68.41 & 29.01 & 20.48 & 40.16 & 39.34 \\
        Lower-bound & 47.06 & 42.46 & 43.78 & 64.07 & 30.51 & 21.21 & 37.32 & 37.09 \\ \midrule
        No Defense (FGSM attack) & 26.80 & 27.21 & 29.05 & 36.41 & 16.44 & 12.05 & 22.99 & 21.57 \\
        No Defense (C\&W attack) & 34.53 & 35.66 & 35.54 & 56.59 & 24.27 & 13.37 & 34.10 & 29.80 \\
        No Defense (PGD attack) & 22.50 & 19.63 & 15.42 & 15.33 & 9.18 & 8.29 & 22.72 & 14.34 \\
        \bottomrule
    \end{tabular}
    }
    \vspace{-4mm}
\end{table*}

\subsection{Experimental Setup}

\textbf{Datasets and Evaluation Metrics.} In our experiments, we leverage V2X-Sim \cite{liV2XSimMultiAgentCollaborative2022} \textcolor{mycolor}{and DAIR-V2X \cite{yuDAIRV2XLargeScaleDataset2022} as our datasets, which are simulated and real-world datasets, respectively.} In addition, to evaluate the performance of the object detection and BEV segmentation tasks, we adopt average precision (AP) with Intersection over Union (IoU) thresholds 0.5 and 0.7 (AP@0.5 and AP@0.7) for object detection task, and mean Intersection over Union (mIoU) for BEV segmentation task. We also use Verification Count to evaluate the performance of PASAC, which is the total number of times that malicious agents are checked.

\textbf{Implementation Details.} We leverage the backbone which is the same as \cite{8578474} for the object detection task, and use U-Net \cite{ronnebergerUNetConvolutionalNetworks2015} as the backbone for the BEV segmentation task. The fusion method for the both tasks is V2VNet \cite{10.1007/978-3-030-58536-5_36}.
Our experiment is deployed on a computer consisting of 2 Intel(R) Xeon(R) Silver 4410Y CPUs (2.0GHz), four NVIDIA RTX A5000 GPUs, and 512GB DDR4 RAM.
As for the implementation of adversarial attacks, we employ three kinds of attacks: fast gradient sign method (FGSM) \cite{goodfellow2015explainingharnessingadversarialexamples}, Carlini \& Wagner (C\&W) \cite{carlini2017evaluatingrobustnessneuralnetworks}, and the projected gradient descent (PGD) \cite{madry2018towards}. For each attack, 
we set the maximum perturbation $\delta_\mathrm{max} = 0.1$, iterations steps $T = 15$, and the step size $\gamma = 0.01$.

\begin{figure*}[t]
    \centering
    \includegraphics[width=1\linewidth]{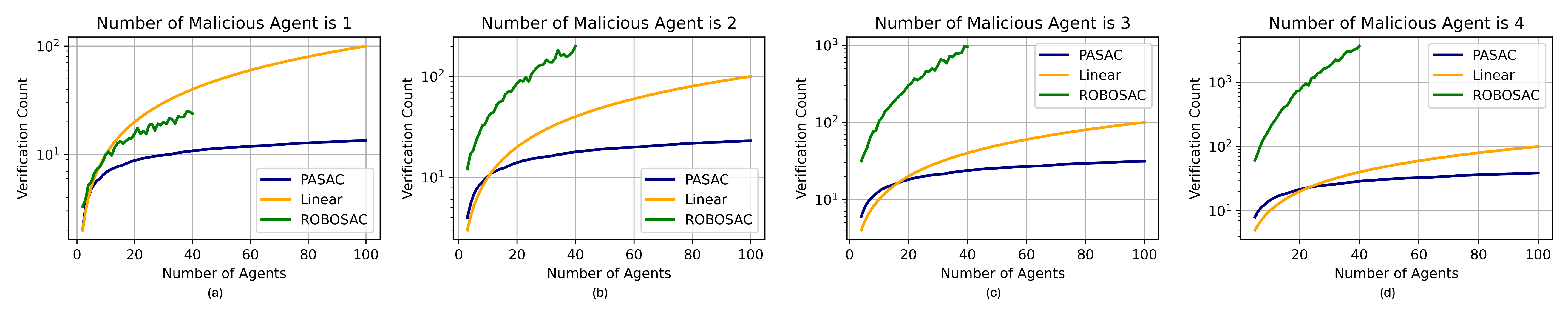}
    \vspace{-8mm}
    \caption{\textbf{Comparison results of PASAC, ROBOSAC and Linear Sampling.} The y-axis represents the verification count, which is in logarithmic scale.}
    \label{fig:pasac_analysis_4}
    \vspace{-5mm}
\end{figure*}

\subsection{Evaluation and Analysis}
\subsubsection{Evaluation of CP-uniGuard.} 
We evaluate the efficacy of our CP-uniGuard scheme against a variety of adversarial attacks. The outcomes of these evaluations are detailed in Table \ref{tab:quantitative_results}, Table \ref{tab:quantitative_results_object_detection}, and Table \ref{tab:performance_comparison_dair_v2x}. As shown in Table \ref{tab:quantitative_results}, in scenarios where the CP system lacks defensive mechanism, the mIoU across all three attack modalities significantly falls below the established lower bound, registering at 37.09\%. This substantial degradation in performance underscores the effectiveness of the adversarial attacks implemented. Conversely, our CP-uniGuard framework can effectively 
defend against these attacks and achieve an mIoU that closely approaches the upper bound of 40.45\%.

We also evaluate our approach on object detection (Table~\ref{tab:quantitative_results_object_detection}). CP-uniGuard consistently achieves the best AP@0.5 and AP@0.7 under all attack scenarios. For example, against PGD, it obtains 80.4 (AP@0.5) and 78.3 (AP@0.7), outperforming ROBOSAC by 2.5 and 2.7 percentage points, respectively. For C\&W attacks, the improvements over ROBOSAC are even larger (5.7 and 6.5 percentage points).

\textcolor{mycolor}{Finally, we evaluate our method in the real-world dataset, DAIR-V2X (Table~\ref{tab:performance_comparison_dair_v2x}). As shown in the table, our method achieves the best performance among all compared defense methods, with AP@0.5 and AP@0.7 scores of 58.31\% and 55.32\%, respectively. This not only significantly outperforms both ROBOSAC and adversarial training, but also closes most of the gap to the upper bound where no attack is present. These results further validate the effectiveness and robustness of our approach on real-world datasets.
}
\begin{table}[t]
    \caption{\textbf{Quantitative results on object detection task in V2X-Sim dataset.} The upper-bound represents collaborative perception with all benign agents, while the lower-bound corresponds to individual perception.}
    \label{tab:quantitative_results_object_detection}
    \vspace{-3mm}
    \resizebox{1\linewidth}{!}{
    \begin{tabular}{l|cc}
    \toprule
    \textbf{Method} & \textbf{AP@0.5} & \textbf{AP@0.7} \\ \midrule
    Upper-bound++ & 81.8 & 79.6 \\
    \midrule
    
    PGD Trained (White-box Defense) & 75.6 & 73.0 \\ 
    ROBOSAC (against PGD attack) & 77.9 & 75.6 \\
    CP-uniGuard (against PGD attack) & \textbf{80.4} & \textbf{78.3} \\
    \midrule
    
    C\&W on PGD Trained (Black-box Defense) & 43.2 & 40.8 \\ 
    ROBOSAC (against C\&W attack) & 74.5 & 71.1 \\
    CP-uniGuard (against C\&W attack) & \textbf{80.2} & \textbf{77.6} \\
    
    \midrule
    Lower-bound & 64.1 & 62.0 \\
    No Defense (PGD attack) & 44.2 & 43.7 \\ \bottomrule
    \end{tabular}
    }
\end{table}

\begin{table}[t]
\centering
\caption{\textbf{Quantitative results on DAIR-V2X dataset.}}
\label{tab:performance_comparison_dair_v2x}
\vspace{-3mm}
\resizebox{1\linewidth}{!}{
\begin{tabular}{p{5cm}|c|c}
\toprule
\textbf{Method} & \textbf{AP@0.5} & \textbf{AP@0.7} \\
\midrule
No Defense (PGD Attack) & 36.55 & 34.12 \\
AdvTrain & 50.21 & 47.45 \\
ROBOSAC & 42.19 & 40.52 \\
CP-uniGuard (ours) & \textbf{58.31} & \textbf{55.32} \\
\midrule
Upperbound & 59.43 & 56.19 \\
\bottomrule
\end{tabular}
}
\vspace{-5mm}
\end{table}

\subsubsection{Evaluation of PASAC.} To investigate the performance of PASAC, we conduct extensive experiments to study the relationship between the verification count and the number of benign agents and malicious agents. We plot two set of figures which are shown in Fig. \ref{fig:pasac_analysis} and Fig. \ref{fig:pasac_analysis_2}.
In Fig. \ref{fig:pasac_analysis}, the x-axis represents the number of benign agents and the y-axis represents the verification count. There are three lines in each subfigure, which represent the minimum, average, and maximum verification count, respectively. We can observe that the verification count increases with the number of collaborative agents and the growth trend is fast at the beginning and then becomes slow. In addition, the verification count is far less than the total number of agents, which indicates that PASAC is efficient in sampling collaborators.
In addition, Fig. \ref{fig:pasac_analysis_2} study the number of malicious agents vs verificaton count with fixed number of  overall agents. We further compare the sample efficiency of different methods in the following section.

\begin{table}[t]
    \caption{\textbf{Comparison results} between ROBOSAC and PASAC. In this experiment, the attack ratios are known to the ROBOSAC method.}
    \label{tab:comparison_results}
    \vspace{-3mm}
    \resizebox{1\linewidth}{!}{
    \begin{tabular}{c|ccc|ccc}
    \toprule
     & \multicolumn{3}{c|}{ROBOSAC} & \multicolumn{3}{c}{PASAC (Ours)} \\ \midrule
    \multirow{2}{*}{Attack Ratio}& \multicolumn{3}{c|}{Verification Count} & \multicolumn{3}{c}{Verification Count} \\
                                  & Min & Max & Avg & Min & Max & Avg \\ \midrule
    0.8 & 1 & 17 & 4.73 & 8 & \textbf{8} & 8.00 \\ 
    0.6 & 1 & 46 & 8.29 & 6 & \textbf{8} & \textbf{7.59} \\ 
    0.4 & 1 & 39 & 10.36 & 4 & \textbf{8} & \textbf{6.60} \\ 
    0.2 & 1 & 19 & 4.89 & 4 & \textbf{6} & \textbf{4.79} \\ \midrule
    Average & 1.00 & 30.25 & 7.06 & 5.50 & \textbf{7.50} & \textbf{6.74} \\ \bottomrule
    \end{tabular}
    }
    \vspace{-1mm}
\end{table}

\begin{table}[t]
    \centering
    \caption{\textbf{Comparison of group checking latency.}}
    \label{tab:latency}
    \vspace{-3mm}
    \resizebox{1\linewidth}{!}{
    \begin{tabular}{c|c|c}
    \toprule
    Num Agent & Method & Latency per group checking (s)   \\
    \midrule
    \multirow{3}{*}{5}         & ROBOSAC &1.62     \\
              & PASAC &1.71     \\
              & Linear &1.55     \\
    \midrule
    \multirow{3}{*}{4}         & ROBOSAC &1.42     \\
              & PASAC &1.47     \\
              & Linear &1.34     \\
    \midrule
    \multirow{3}{*}{3}         & ROBOSAC &1.02     \\
              & PASAC &1.07     \\
              & Linear &1.03     \\
    \bottomrule
    \end{tabular}
    }
    \vspace{-5mm}
    \end{table}

\subsubsection{Comparison Results in Sampling Efficiency.}
We evaluate the sampling efficiency of PASAC compared to the previous state-of-the-art, ROBOSAC \cite{liUsAdversariallyRobust2023}, following its experimental setup with known attack ratios. As shown in Table~\ref{tab:comparison_results}, PASAC consistently requires fewer verification steps than ROBOSAC across all attack ratios, and its results are notably more stable (e.g., much lower maximum verification count).
We further compare PASAC, ROBOSAC, and Linear Sampling without prior knowledge of malicious agent ratios. As shown in Fig.~\ref{fig:pasac_analysis_4}, PASAC consistently requires fewer verification steps than the others, while ROBOSAC's efficiency sharply drops when attack ratios are unknown. This highlights PASAC's superior sampling efficiency and robustness.
\vspace{-3mm}

\begin{figure*}[t]
    \centering
    \begin{subfigure}{0.45\textwidth}
        \includegraphics[width=\linewidth]{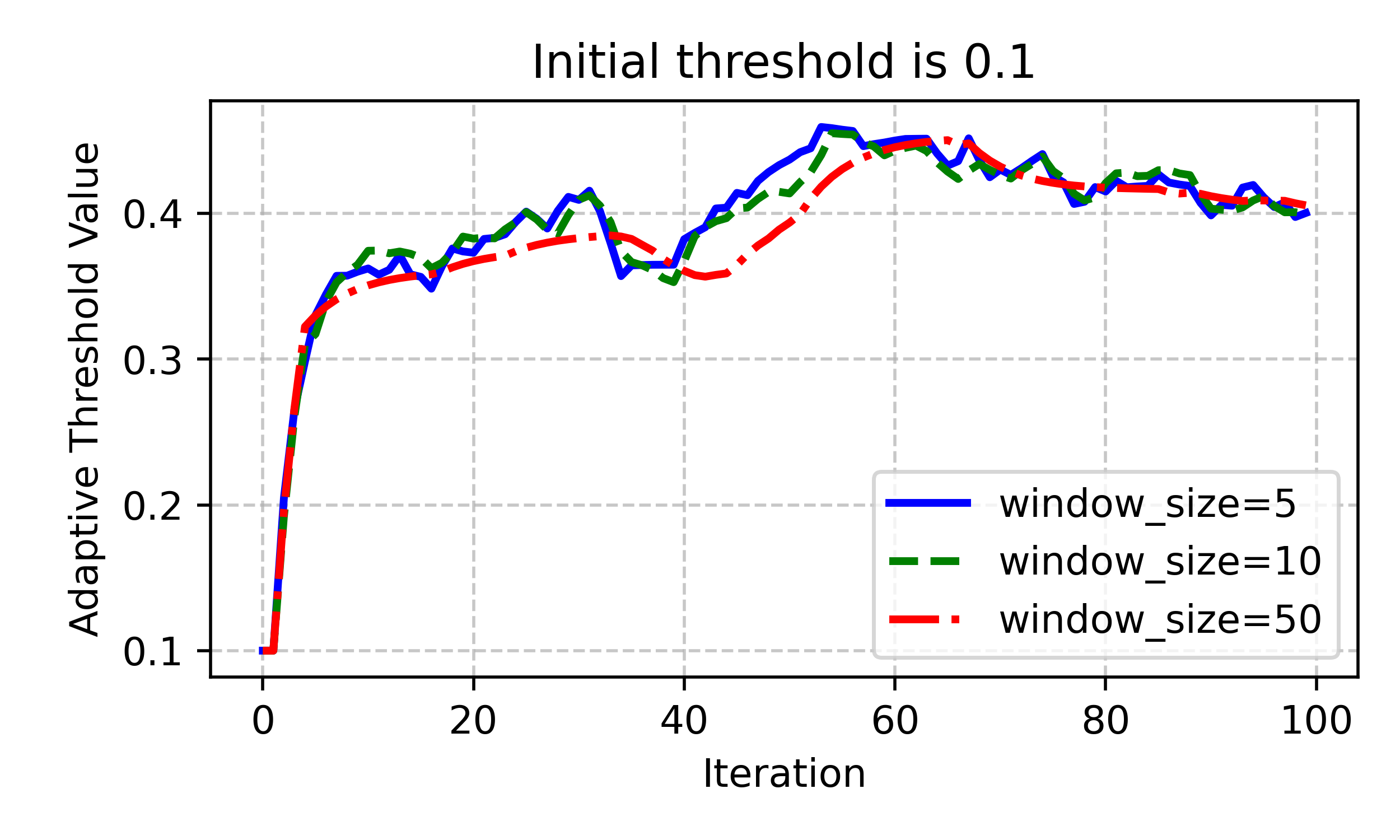}
        \vspace{-2em}
        \caption{}
    \end{subfigure}
    \begin{subfigure}{0.45\textwidth}
        \includegraphics[width=\linewidth]{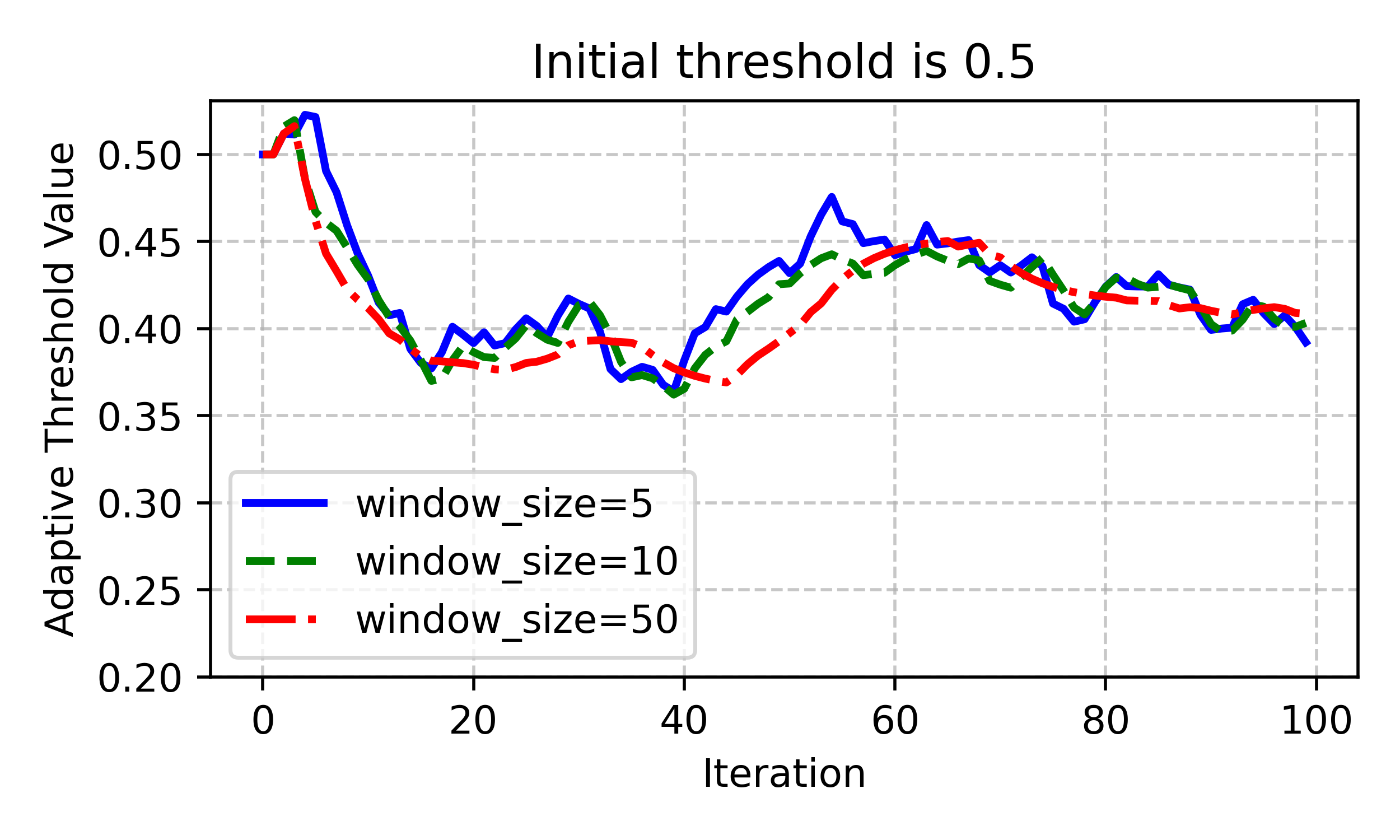}
        \vspace{-2em}
        \caption{}
    \end{subfigure}
    
    \begin{subfigure}{0.45\textwidth}
        \includegraphics[width=\linewidth]{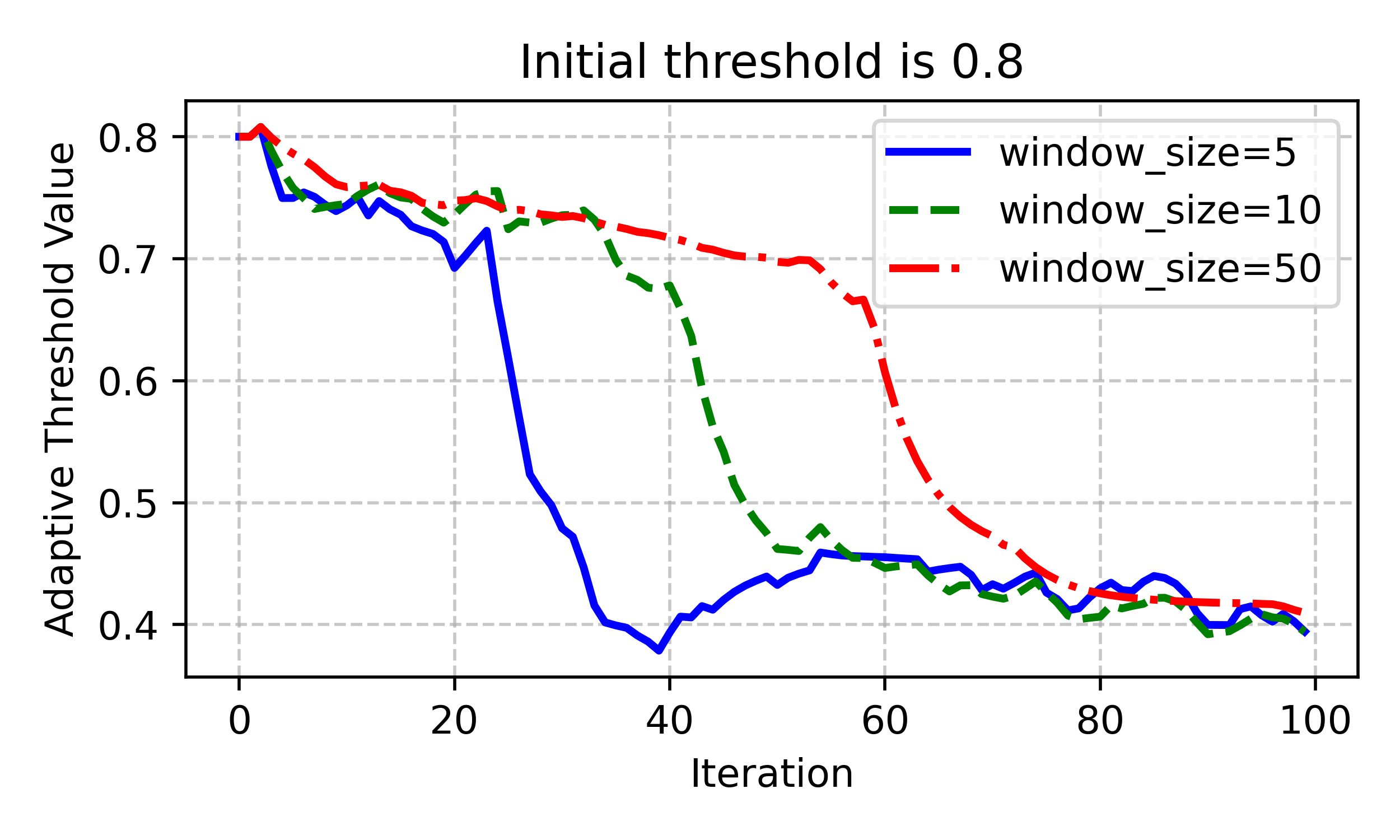}
        \vspace{-2em}
        \caption{}
    \end{subfigure}
    \begin{subfigure}{0.45\textwidth}
        \includegraphics[width=\linewidth]{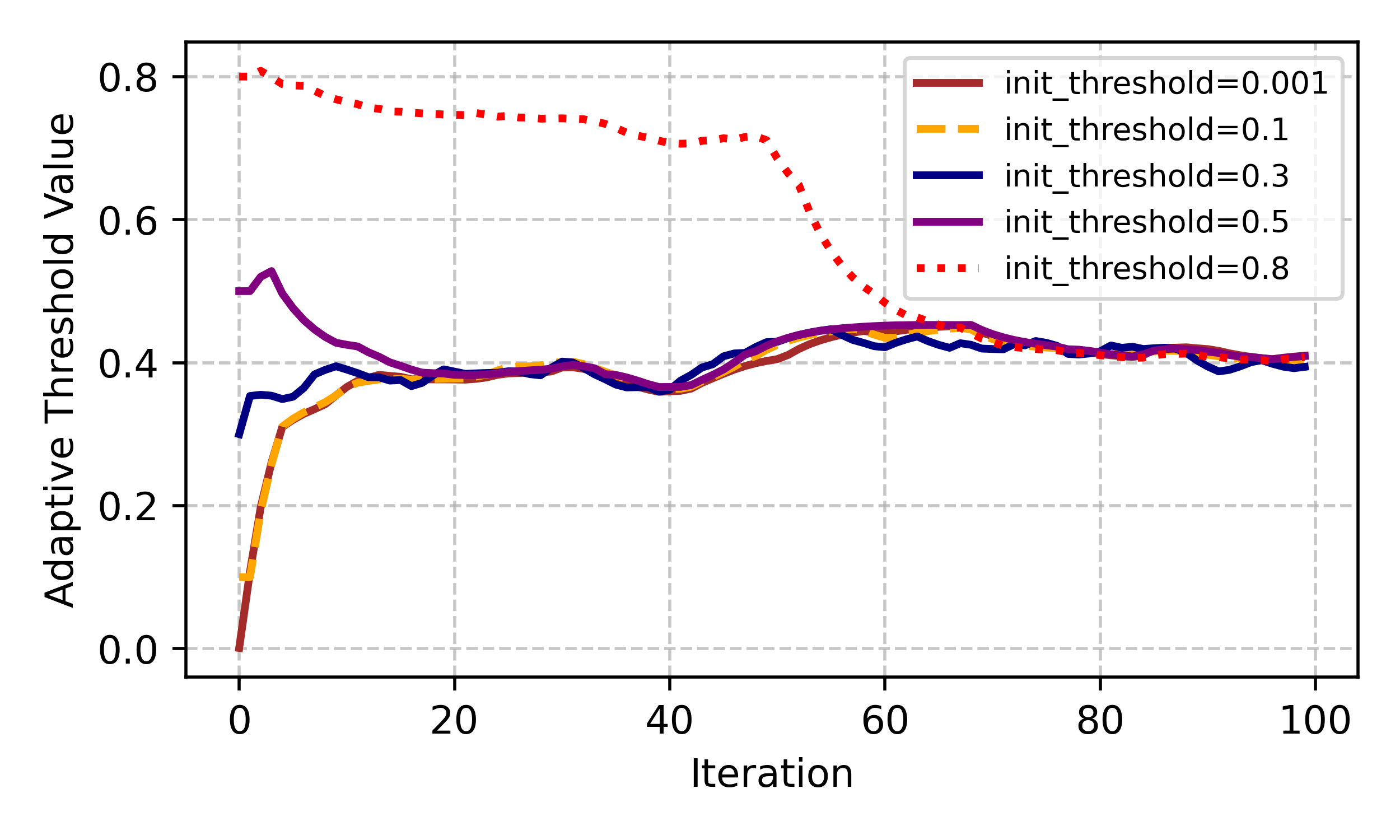}
        \vspace{-2em}
        \caption{}
    \end{subfigure}
    \vspace{-0.5em}
    \caption{Ablation study of the adaptive threshold adjustment dynamics for different initial threshold values and window sizes.}
    \vspace{-3mm}
    \label{fig:window_comparison}
\end{figure*}
    
    \begin{table}[t]
        \centering
        \caption{Comparison between CCLoss and Mainstream Loss Functions}
        \label{tab:r1_loss_comparison}
        \vspace{-3mm}
        \begin{tabular}{l|c|c}
        \toprule
        Detection Loss & AP@0.5 & AP@0.7 \\
        \midrule
        Smooth L1 \cite{pytorch_smoothl1loss} & 49.64 & 48.79 \\
        IoU \cite{8886046} & 66.52 & 64.92 \\
        CCLoss (ours) & \textbf{80.41} & \textbf{78.35} \\
        \midrule[0.75pt]
        Segmentation Loss & \multicolumn{2}{c}{AP@IoU} \\
        \midrule
        Cross Entropy \cite{10.5555/3618408.3619400} & \multicolumn{2}{c}{37.09} \\
        CCLoss (ours) & \multicolumn{2}{c}{\textbf{39.34}} \\
        \bottomrule
        \end{tabular}
        \vspace{-3mm}
    \end{table}

\textcolor{mycolor}{\subsubsection{Comparison Results in Computational Latency.}
We measured the average latency per group checking\footnote{This means the latency that verifies one collaborative group agents to identify the malicious agent. For example, with collaborative 5 agents, the latency of identifying the malicious agent from the 5 agents is used to denote this metric.} for ROBOSAC, PASAC, and Linear checking, the results are shown in Table~\ref{tab:latency}. It demonstrates that the computational latency per group checking is roughly comparable across all methods. Since hardware constraints often necessitate processing agents in groups rather than all-at-once, PASAC's ability to efficiently identify malicious agents through logarithmic search steps (fewer total checking times in large-scale networks) offers a better trade-off between efficiency and robustness compared to checking every agent individually.}

\vspace{-3mm}

\textcolor{mycolor}{
\subsubsection{Comparison Results of CCLoss.}
In addition, to further illustrate the effectiveness of our CCLoss, we conduct additional experiments to
compare the CCLoss with mainstream loss functions, including Smooth L1 loss \cite{pytorch_smoothl1loss} and IoU loss \cite{8886046} for object
detection, and Cross Entropy loss \cite{10.5555/3618408.3619400} for segmentation. The results are shown in Table \ref{tab:r1_loss_comparison}. We can see that
CCLoss signiﬁcantly outperforms the mainstream loss functions across both detection and segmentation
tasks. These results demonstrate that our CCLoss not only provides a stronger veriﬁcation criterion for
consensus but also contributes directly to superior perception performance when compared to popular,
widely-used alternatives. 
}

\begin{table}[t]
    \centering
    \caption{Sensitivity analysis of the initial threshold value of the online adaptive threshold.}
    \vspace{-3mm}
    \label{tab:threshold_sensitivity}
    \begin{tabular}{c|c|c}
    \toprule
    Initial Threshold & AP@0.5 & AP@0.7 \\
    \midrule
    0.001 & 79.62 & 76.96 \\
    0.05 & 79.31 & 77.01 \\
    0.1 & 79.50 & 76.77 \\
    0.5 & 80.16 & 78.36 \\
    0.8 & 76.65 & 75.17 \\
    \bottomrule
    \end{tabular}
    \vspace{-3mm}
    \end{table}

\begin{figure*}[t]
    \centering
    \includegraphics[width=.8\linewidth]{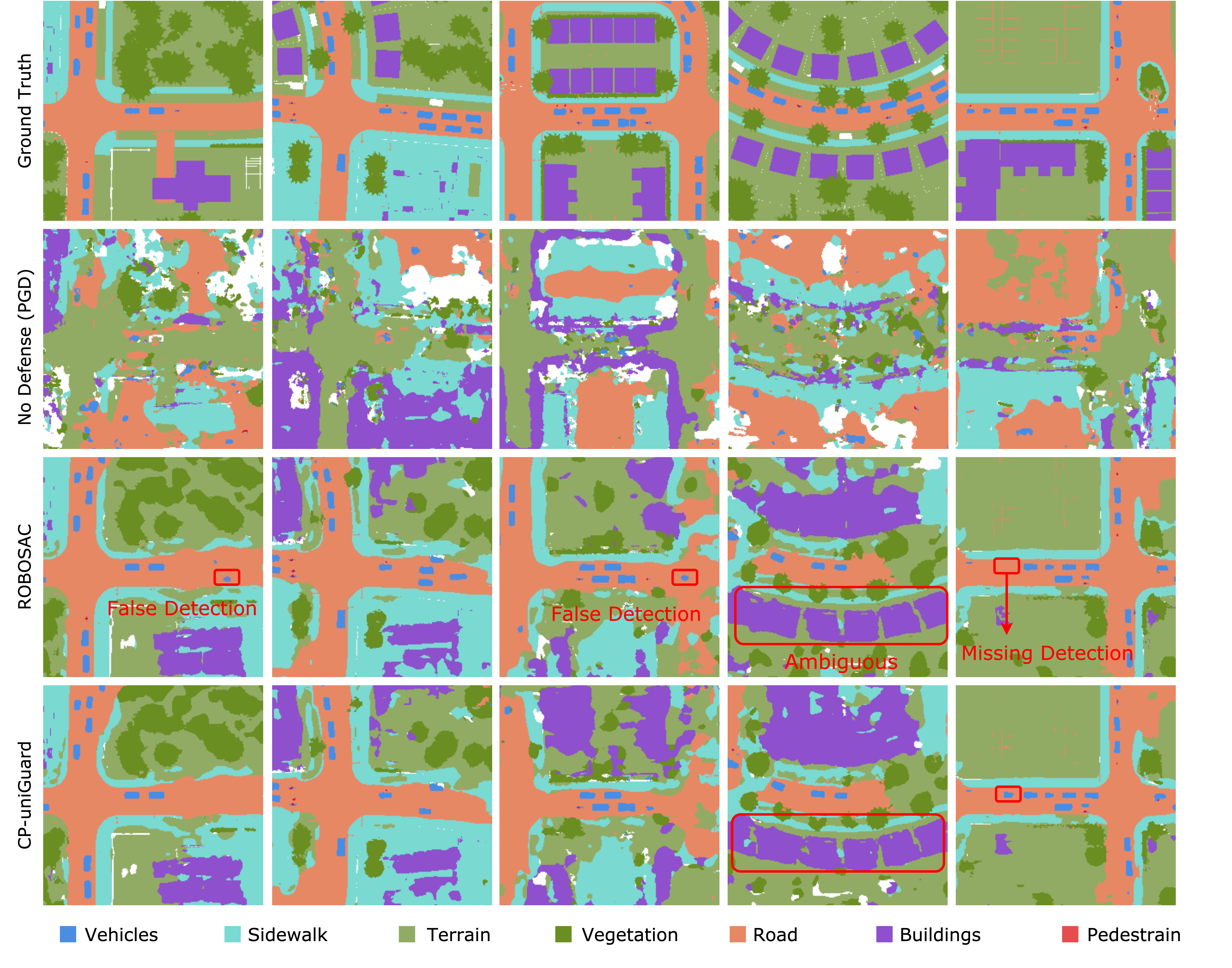}
    \vspace{-3mm}
    \caption{\textbf{Visualization} of no defense and defensive CP-uniGuard and ROBOSAC results on V2X-Sim datasets.}
    \label{fig:cp_visualization_3}
    \vspace{-5mm}
\end{figure*}

\subsubsection{Study on the Online Adaptive Threshold Mechanism.}

In this section, we conducted an analysis on the online adaptive threshold mechanism. As shown in Fig. \ref{fig:window_comparison} (d), we plot the adaptive threshold value with different initial thresholds. We can observe that the threshold can quickly converge to a stable value, which is around 0.4. In addition, the threshold is not fixed, but dynamically adjusted based on the current situation to ensure the robustness of the system and a low false positive rate. \textcolor{mycolor}{Fig. \ref{fig:window_comparison} (a), (b), and (c) also indicate that the proposed mechanism consistently converges to a stable value regardless of initialization, and the window size controls the sensitivity of the threshold to the recent history of consistency scores, demonstrating robustness and fast adaptation. }

\textcolor{mycolor}{
In addition, Table~\ref{tab:threshold_sensitivity} presents results where we vary the initialization of the adaptive threshold across a wide range and report the corresponding performance. It shows that our framework maintains robustness to different initial values, as the average precision only varies within a small margin and the overall performance remains strong. This stable convergence to an effective value demonstrates that our adaptive adjustment mechanism mitigates the effect of this hyperparameter and further validates the robustness of our method with respect to threshold selection.}

\begin{table}[t]
    \caption{\textbf{Generalization Evaluation.} We evaluate and compare the generalization capability of CP-uniGuard on two attacks, PGD and C\&W, respectively.}
    \label{tab:generalization_evaluation}
    \vspace{-3mm}
    \resizebox{1\linewidth}{!}{
    \begin{tabular}{l|cc}
    \toprule
    \textbf{Method} & \textbf{AP@0.5} & \textbf{AP@0.7} \\ \midrule

    CP-Guard+ (against PGD attack) & 71.7 & 68.6 \\
    CP-uniGuard (against PGD attack) & \textbf{80.4 (+8.7)} & \textbf{78.3 (+9.7)}\\
    \midrule
    
    CP-Guard+ (against C\&W attack) & 73.9 & 71.2 \\
    CP-uniGuard (against C\&W attack) & \textbf{80.2 (+6.3)} & \textbf{77.6 (+6.4)} \\
      \bottomrule 
    \end{tabular}
    }
    \vspace{-3mm}
\end{table}

\subsubsection{Generalization Evaluation}

We further evaluate the generalization capability of CP-uniGuard across different attack types, comparing it with feature-level defense methods such as CP-Guard+ \cite{hu2025cpguardnewparadigmmalicious}. For CP-Guard+, we train the model on the CP-GuardBench dataset \cite{hu2025cpguardnewparadigmmalicious} while excluding specific attacks from the training set to assess its generalization performance. Specifically, we exclude the PGD attack and C\&W attack during training and subsequently evaluate the model's performance against these two attacks, respectively. Results are presented in Table~\ref{tab:generalization_evaluation}, we can see that CP-uniGuard consistently achieves superior performance compared to CP-Guard+ under both PGD and C\&W attacks which were unseen. This performance gap stems from CP-uniGuard's fundamental design: rather than relying on training, CP-uniGuard leverages consensus verification, which provides inherent generalization to unseen attacks without requiring prior exposure during training.

\subsubsection{Qualitative Evaluation}

As depicted in Fig. \ref{fig:cp_visualization_3}, we present the visualization results on the V2X-Sim dataset. Without CP-uniGuard, attackers can significantly disrupt collaborative perception, leading to a marked degradation in the performance of BEV segmentation tasks.
However, our introduced CP-uniGuard framework can intelligently identify benign collaborators and eliminate malicious collaborators, thereby facilitating robust CP. \textcolor{mycolor}{In addition, compared with ROBOSAC which generates false positives and segmentation ambiguity, CP-uniGuard can achieve more accurate and stable results.}

\section{Limitations and Future Work}
\label{sec:limitations_future_work}

Despite the promising results of CP-uniGuard in defending against malicious agents in collaborative perception, our method still has some limitations and failure cases: 
\textcolor{mycolor}{
\begin{enumerate}\setlength{\itemsep}{0pt}
    \item \textit{Adaptive Collusion in Blind Spots}: If adaptive attackers coordinate to generate ``consistent" false information (e.g., ghost objects) specifically located in the ego agent's blind spot, they may evade the spatial consistency checks.
    \item \textit{Adaptive Threshold Exploitation}: While our online threshold $\varepsilon_t$ is dynamic, a sophisticated attacker with perfect knowledge of ego agent's recent history could theoretically adapt their perturbation magnitude to stay marginally below the instantaneous rejection threshold.
    \item \textit{Ego Sensor Failure}: Since CP-uniGuard uses the ego agent as the root of trust, if an ego's own sensors are severely compromised (e.g., by physical obstruction), the system may incorrectly flag benign collaborators as malicious due to the high discrepancy.
\end{enumerate}
In our future work, we plan to address these limitations by enhancing the temporal and intrinsic robustness of the framework. To counter adaptive collusion in blind spots, we will incorporate multi-frame temporal consistency checks, leveraging historical trajectory data to distinguish valid dynamic agents from coordinated ghost objects that lack plausible motion continuity. To mitigate the risk of white-box threshold exploitation, we aim to introduce randomization mechanisms or long-term behavioral monitoring, enabling the detection of adversaries that persistently hover dangerously close to the rejection boundary. Finally, to address the single point of failure regarding ego sensor health, we will integrate intrinsic sensor health monitoring and cross-modal verification (e.g., LiDAR-Camera consistency), allowing the system to self-diagnose reliability issues and dynamically adjust confidence levels before penalizing collaborators.}

\section{Conclusion}
\label{sec:conclusion}
In this paper, we have designed a novel defense framework for CP named CP-uniGuard, which consists of three parts. The first is PASAC which can effectively sample the collaborators without the prior probabilities of malicious agents. The second is collaborative consistency loss verification which calculates the discrepancy between the ego agent and the collaborators, which is used as a verification criterion for consensus. The third is the online adaptive threshold which can adaptively adjust the threshold to ensure the stability of the system in dynamic environments.
Extensive experiments show that our CP-uniGuard can defend against different types of attacks and can adaptively adjust the trade-off between performance and computational overhead.


\bibliographystyle{IEEEtran}

{\small
\bibliography{ref, ref2}}




\end{document}